\documentclass[runningheads]{llncs}

\usepackage{eccv}

\usepackage{eccvabbrv}

\usepackage{graphicx}
\usepackage{wrapfig}
\usepackage{booktabs}

\usepackage[accsupp]{axessibility}  

\usepackage{hyperref}

\usepackage{orcidlink}
\usepackage[numbers]{natbib}
\usepackage{mathtools}

\newcommand{\smjacmat}[1]{\mathcal{M}(#1)}
\newcommand{\smjac}[1]{\mathcal{J}_{\texttt{sm}}(#1)}

\def\diag{\operatorname{diag}}
\def\blkdiag{\operatorname{bdiag}}

\newtheorem{prop}{Proposition}

\begin{document}

\title{Pay Attention to Attention Distribution: A New Local Lipschitz Bound for Transformers} 

\titlerunning{Pay Attention to Attention Distribution}

\author{Nikolay Yudin \inst{1} \and Sergei Kudriashov \inst{1} \and Alexander Gaponov \inst{1} \and Maxim Rakhuba \inst{1}}

\authorrunning{N.~Yudin et al.}

\institute{HSE University \\ \email{neyudin@edu.hse.ru}}

\maketitle

\begin{abstract}
We introduce a novel upper bound on the local Lipschitz constant of the dot-product self-attention block showing its dependence on the attention map distributions.
The proposed bound is not only tighter than the prior art, but for the first time, reveals how the distribution of attention probabilities shapes the local Lipschitz constant of the self-attention block. 
The theoretical basis of the proposed upper bound lies in the refined closed-form upper bounds on singular values of the  Jacobian of softmax function.
Leveraging these theoretical insights, we introduce \texttt{JaSMin} (\textbf{Ja}cobian \textbf{S}oftmax norm \textbf{Min}imization), a lightweight regularizer that directly controls the local Lipschitz constant of each block and, consequently, the entire model. Additionally, we discuss how the nature of the attention map distribution contributes to the gradient dynamics and, consequently, transformer training stability.
\end{abstract}

\section{Introduction}

In this work, we analyze the spectral norm of the Jacobian of dot-product self-attention \citep{vaswani2017attention}, which governs its local Lipschitz properties. 
While prior works attempted to bound the local Lipschitz constant of self-attention \citep{kim2021lipschitz, dasoulas2021lipschitz, castin2023smooth, hu2024specformer}, they disregarded the structure of the attention score matrix, proposing bounds that were dependent only on weights' and inputs' norms. We consider this fact to be a significant limitation of these bounds as it naturally omits the powerful component from the analysis. In our work, we pay a particular attention to this problem, emphasizing the dependency of the self-attention local Lipschitz constant on attention map distribution.
In particular, we establish a novel theoretical upper bound on the local Lipschitz constant of self-attention that is based on the new refined spectral properties of the softmax Jacobian.

Overall, contributions of this work can be summarized as follows:
\begin{itemize}
    \item We introduce interlacing bounds for \textit{all} singular values of the softmax Jacobian, improving eigenvalue interlacing theorem \citep{horn2012matrix, tyrtyshnikov1997brief}  (see \Cref{th:sm_jac_bound}). 
    For the largest singular value, we prove that our bound surpasses the bound based on the Gershgorin theorem \citep{nair2025softmax, Huang_Farahmand_Kitani_Bagnell_2015}.
    \item We introduce an improved upper bound on the spectral norm of the self-attention Jacobian, which captures the nature of attention map scores (see \Cref{th:attn_jac_bound}). We compare our bound with known estimates and show that our bound is tighter than prior art for different attention-based vision models.
    \item Using the proposed bounds on the softmax Jacobian, we show that the gradient vanishes for both the categorical and uniform attention map distributions. In either extreme, this signal decay hinders optimization for these attention heads and layers, linking these phenomena to the softmax 
    Jacobian spectrum. (see \Cref{sec:gradient_sensitivity})
    \item We present \texttt{JaSMin} (\textbf{Ja}cobian \textbf{S}oftmax norm \textbf{Min}imi\-zation), a regularization method aimed to control the self-attention Jacobian spectral norm and quantify gradient sensitivity (see \Cref{sec:lip_const_control}).
    We show that our regularizer reduces the spectral norm of the whole model Jacobian, supporting our theory. In practice \texttt{JaSMin} leads to the increase of adversarial metrics on CIFAR \citep{krizhevsky2009learning} and Imagenette \citep{imagenette} datasets (see \Cref{sec:experiments}).
\end{itemize}

\section{Related Work}

The dot-product self-attention \citep{vaswani2017attention} has become a dominant part of most transformer models, serving as a default option for production-grade large-scale models \citep{jiang2023mistral, yang2025qwen3technicalreport}. 
Despite its popularity, the theoretical understanding of the fundamental properties of this operation is still incomplete. 
Recently, it has been observed that dot-product attention is susceptible to fluctuations, which has a negative impact on the model's performance, especially apparent for long sequences, and has motivated researchers to seek for sparse attention alternatives \citep{ye2025differential, jiang2024minference}.
Another noticeable pattern is an appearance of ``attention sinks'' \citep{xiao2023efficient, gu2025attentionsinkemergeslanguage, feng2025edit0}, when an excessive amount of attention mass allocated to certain tokens, which turn out to serve as ``sinks'' for the excess attention, when further modifications to the residual stream are no longer necessary. 
On the other hand, \citep{dong2021attention} has argued that pure dot-product attention possesses a strong bias towards ``token uniformity'', which leads to the rank collapse and oversmoothing in deep models \citep{shi2022revisitingoversmoothingbertperspective}. While layer normalization \citep{wu2024roleattentionmaskslayernorm} and higher Lipschitzness of MLP layers were shown to stabilize the self-attention dynamics and prevent the doubly exponential convergence to rank-1 representations, the theoretical understanding of interconnections between these phenomena remains limited. 

In the visual domain, the rapid adoption of Vision Transformers (ViTs) \citep{dosovitskiy2020image} has prompted parallel investigations into their architectural robustness, sensitivity, and efficiency. To address the fundamental limitations of standard dot-product attention, architectures like ShiftViT \citep{wang2022shift} and LipShiFT \citep{menon2025lipshift} proposed abandoning attention entirely in favor of lightweight shifting, while Robust Vision Transformer \citep{mao2022towards} introduced position-aware attention scaling to improve inherent resistance to corruptions. Beyond the scope of adversarial robustness, foundational modifications for specific tasks and computational regimes were proposed: Swin Transformer \citep{liu2021swin} localizes attention to shifted windows, improving computational efficiency; CViT \citep{wang2024cvit} switched to encoder-decoder transformers for PDE systems, while EA-ViT \citep{zhu2025ea} adapted models for resource-constrained deployment.

In this work we study attention-based vision models stability problem through the lens of neural network Lipschitz constant.
It is known that vulnerability to adversarial attacks is closely linked to Lipschitz continuity and the Lipschitz constant of the neural network \citep{murdock1999perturbations}. 
Despite its successful widespread use, the dot-product self-attention is not globally Lipschitz \citep{kim2021lipschitz}. 
Such a functional property may theoretically cause training instability and poor robustness \citep{mao2022understanding, zhou2022understanding}.

Recent works that analyze the self-attention's local Lipschitz constant \citep{kim2021lipschitz, hu2024specformer, castin2023smooth, dasoulas2021lipschitz} are aimed at bounding the local Lipschitz constant without taking the attention map into account, building bounds dependent on separate norms of inputs and model weights. We argue that, being a core element of self-attention, softmax largely contributes to the local Lipschitz constant bound, motivating us to explore the local Lipschitz constant dependency from attention map scores.

Multiple prior works considered softmax to be $1$-Lipschitz \citep{hu2024specformer, gao2017properties, NEURIPS2018_6a4d5952, gouk2021regularisation}. We admit that the concurrent work \citep{nair2025softmax} also proposes that global Lipschitz constant of softmax function can be bounded by $\frac{1}{2}$ under any $\ell_p$ norm. Our work focuses on bounding all softmax Jacobian singular values, improving eigenvalue interlacing theorem \citep{horn2012matrix, tyrtyshnikov1997brief} and the result from \citep{nair2025softmax} for the spectral norm.

\section{Notation and Background}
\label{sec:preliminaries}

A neural network can be treated as a function $f\colon \mathcal{X}\subseteq \mathbb{R}^D \to \mathbb{R}^d$, mapping input data to corresponding outputs. A typical approach to quantify its sensitivity to input perturbations is to analyze its Lipschitz continuity.
A Lipschitz constant of $f$ over an open set $\mathcal{X}\subseteq \mathbb{R}^D$ with respect to Euclidean norm $\|\cdot\|_2$ can be defined as follows:
\begin{equation}
    \texttt{Lip}(f; \mathcal{X})=\sup_{\substack{x_1, x_2 \in \mathcal{X}\\ x_1 \neq x_2}} \frac{\|f(x_1)-f(x_2)\|_2}{\|x_1-x_2\|_2}.
\end{equation}
Rather than considering the global Lipschitz continuity (where $\mathcal{X} = \mathbb{R}^D$), we focus on the analysis of the \emph{local} Lipschitz constants, which is less computationally prohibitive.
Specifically, we examine the behavior of $f$ in the neighborhood of balls $B_\varepsilon(x_0)$ of a radius $\varepsilon>0$ centered at various data points $x_0$.
Assuming differentiability and denoting the Jacobian of $f(x)$ by $\mathcal{J}_f(x)\in\mathbb{R}^{d\times D}$, for each such ball we may write:
\begin{equation}
    \texttt{Lip}(f;\mathcal{X}) = \sup_{x\in \mathcal{X}} \|\mathcal{J}_f(x)\|_2,
\end{equation}
where $\|\cdot\|_2$ for a matrix denotes the spectral norm of a Jacobian matrix, which is equal to its largest singular value.
For notation simplicity, we further write all the theoretical results in terms of the spectral norm of the Jacobian matrix, implicitly omitting the supremum across all data points.

In our paper, we focus on controlling the local Lipschitz constant of separate dot product self-attention blocks. The choice of considering local Lipschitz constant instead of global one is based on the fact that dot-product self-attention is not globally Lipschitz \citep{kim2021lipschitz}. Despite the absence of global Lipschitzness, there is an interest in achieving training stability and analyzing key components of attention-based models making them robust to input perturbations.

Let us define dot product self-attention \citep{vaswani2017attention} $\texttt{Attn}_h: \mathbb{R}^{N \times D} \to \mathbb{R}^{N \times d}$:
\begin{equation}\label{eq:attention}
    \texttt{Attn}_h(X) = \texttt{sm}\left(\frac{XW_h^Q(W_h^{K})^{\top}X^{\top}}{\sqrt{d}}\right)XW_h^V = \texttt{sm}\left(XA_hX^{\top}\right)XW_h^V,
\end{equation}
where $X \in \mathbb{R}^{N \times D}$, $W_h^{Q}, W_h^{K}, W_h^{V} \in \mathbb{R}^{D \times d}$, $\texttt{sm}$ denotes row-wise softmax and
\begin{equation}
    A_h = \frac{W_h^{Q}(W_h^{K})^{\top}}{\sqrt{d}}.
\end{equation}
Additionally, let us denote the attention map matrix $P^h \in \mathbb{R}^{N \times N}$ for a head $h$:
\begin{equation}\label{eq:attn_scored_def}
    P^h = \texttt{sm}(XA_hX^{\top}),
\end{equation}
whose rows we further denote as $P^h_{i, :}$. For simplicity, we omit output matrix $W_h^{O}$.

In \citep{kim2021lipschitz} authors first derived the closed form for the single-head self-attention Jacobian: using the notation above, we may write the Jacobian as a block matrix:
\begin{equation}\label{eq:block_jac}
    \mathcal{J}_{\texttt{Attn}_h}(X) = 
    \begin{bmatrix}
        \mathcal{J}_{\texttt{Attn}_h}(X)_{ij}
    \end{bmatrix}_{i,j=1}^{N,N},
\end{equation}
so that each $\mathcal{J}_{\texttt{Attn}_h}(X)_{ij} \in \mathbb{R}^{d \times D}$ can be expressed in terms of $W_h^V, X, P^h_{i, :}, A_h$:
\begin{equation}\label{eq:attn_jac_block}
\begin{split}
	\mathcal{J}_{\texttt{Attn}_h}(X)_{ij} = (W_h^{V})^{\top}X^{\top}\smjacmat{P^h_{i, :}}(E_{ji}XA_h^{\top} & + XA_h\delta_{ij}) + (W_h^{V})^\top \cdot P^h_{ij},
\end{split}
\end{equation}
where
\begin{itemize}
    \item $E_{ji} \in \mathbb{R}^{N \times N}$ --- a binary matrix with zeros everywhere except $(j, i)$-th place;
    \item $\smjacmat{P^h_{i, :}} = \diag(P^h_{i, :}) - (P^h_{i, :})^{\top} P^h_{i, :}$;
    \item $\delta_{ij} \in \{0, 1\}$ is the Kronecker delta.
\end{itemize}
As a result, the aforementioned Jacobian can be rewritten in a pure matrix form:
\begin{equation}\label{eq:matrix_jac}
\begin{split}
    \mathcal{J}_{\texttt{Attn}_h}(X) = (I_N \otimes (W_h^V)^\top X^\top)\blkdiag\left(\smjacmat{P^h_{1, :}}, \dots,\smjacmat{P^h_{N, :}}\right) \\ (I_N \otimes XA_h^\top + (XA_h \otimes I_N) \mathcal{P}_{N, D}) + P^h \otimes (W_h^V)^\top,
\end{split}
\end{equation}
where $\otimes$ denotes Kronecker product, $\blkdiag(\smjacmat{P^h_{1, :}}, \dots, \smjacmat{P^h_{N, :}})$ denotes a block diagonal matrix with $\smjacmat{P^h_{i, :}}$ on the diagonal and $\mathcal{P}_{N, D}$ is a ``perfect shuffle'' permutation matrix \citep{golub2013matrix} satisfying $\mathrm{vec}(X^\top) = \mathcal{P}_{N, D} \mathrm{vec}(X)$ where $\mathrm{vec}$ is taken in ``c'' order.

Utilizing its structure and refined properties of the softmax Jacobian (see \Cref{th:sm_jac_bound}), we show both theoretically and empirically that the spectral norm of the matrix from \Cref{eq:block_jac} heavily depends on the $\blkdiag(\smjacmat{P^h_{1, :}}, \dots, \smjacmat{P^h_{N, :}})$ (see \Cref{th:attn_jac_bound}).

\section{Refined spectral properties of the softmax Jacobian}
\label{sec:refined_softmax_bound}
Self-attention mechanism relies on the softmax function, which critically influences the structure of its Jacobian. In this section, we derive new properties of the softmax Jacobian --- key to our later analysis of the full self-attention Jacobian. Beyond this immediate application, our results refine previously proposed findings regarding the softmax Jacobian spectral norm \citep{gao2017properties, hu2024specformer, nair2025softmax, NEURIPS2018_6a4d5952, gouk2021regularisation}. Additionally, our theoretical result improves classic eigenvalue interlacing theorem \citep{horn2012matrix,tyrtyshnikov1997brief} for this particular structured matrix.

Recall that the softmax function $\texttt{sm}\colon\mathbb{R}^N \to \mathbb{R}^N$ is defined  as
\begin{equation}
    \texttt{sm}(x)_{i} = \frac{e^{x_i}}{\sum_{j=1}^N e^{x_j}}.
\end{equation}
It is well-known that the softmax Jacobian $\smjac{z}\in\mathbb{R}^{N\times N}$ is a symmetric positive semidefinite matrix of the following form:
\begin{equation} \label{eq:softmaxjac}
    \smjac{z} = \diag\left(\texttt{sm}(z)\right) - \texttt{sm}(z)\texttt{sm}(z)^{\top}.
\end{equation}
For simplicity, we further use the following notation: 
\begin{equation}\label{eq:smjacmat}
\smjacmat{x} = \diag(x) - xx^{\top}    
\end{equation}

Prior to stating the theorem, let us recall key auxiliary results and define notation central to our analysis. First of all, we formulate eigenvalue interlacing theorem \citep{horn2012matrix, tyrtyshnikov1997brief} and recall the results introduced in \citep{nair2025softmax}.
\begin{theorem}[Eigenvalue interlacing theorem \citep{horn2012matrix, tyrtyshnikov1997brief}]\label{th:interlacing}
	Let $A, B \in \mathbb{R}^{n \times n}$ be symmetric matrices and $\alpha_1 \geqslant \dots \geqslant \alpha_n, \, \beta_1 \geqslant \dots \geqslant \beta_n$ be their eigenvalues respectively. If
		$A = B + \varepsilon pp^\top$, $\left\|{p}\right\|_{{2}} = 1$,  $\varepsilon > 0$,
	then
	\begin{equation}
		\alpha_1 \geqslant \beta_1 \geqslant \alpha_2 \geqslant \beta_2 \geqslant \dots \geqslant \alpha_n \geqslant \beta_n.
	\end{equation}
\end{theorem}
Directly applying eigenvalue interlacing theorem to the softmax Jacobian matrix $\smjacmat{\texttt{sm}(z)}$, one may obtain the following interlacing property: for $x = \texttt{sm}(z)$ setting $B = \diag(x) - xx^{\top}, \, A = \diag(x), \, p = x/\|x\|_2, \, \varepsilon=\|x\|_2^2$ we obtain
\begin{equation}
    x_{(1)} \geqslant \sigma_{1}(B)  \geqslant
    x_{(2)} \geqslant \sigma_{2}(B) \geqslant
    \dots \geqslant x_{(N)} \geqslant \sigma_N(B),
\end{equation}
where $x_{(k)}$ here and further is top-$k$ largest vector value.
However, even for the case $N=2$ this bound appears to be loose for the first singular value: in the case of $2$-dimensional probability distribution $p = [1-\varepsilon, \varepsilon], \, 0 < \varepsilon < 1$ for a sufficiently small $\varepsilon$ eigenvalue interlacing theorem gives the following upper bound:
\begin{equation}\label{eq:2dsmjacmat}
    1-\varepsilon \geqslant \left\| \smjacmat{[1-\varepsilon, \varepsilon]}\right\|_2 = \left\|\begin{bmatrix}
        \varepsilon(1 - \varepsilon) & -\varepsilon(1-\varepsilon)\\
        -\varepsilon(1 - \varepsilon) & \varepsilon(1-\varepsilon)
    \end{bmatrix}\right\|_2 = 2\varepsilon(1 - \varepsilon).
\end{equation}

\begin{figure}[h]
    \centering
    \begin{tikzpicture}
    \node (picture) at (0, 0) {\includegraphics[width=1\linewidth]{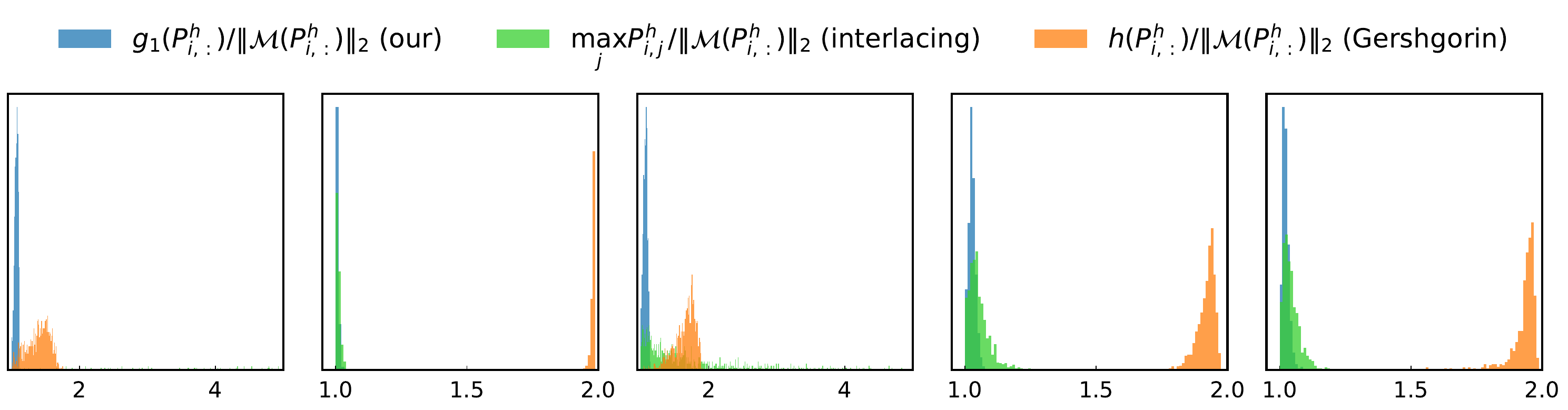}};
    \node[inner sep=0pt] (l4) at (-4.92, -1.6) {\tiny Layer 3, Head 4};
    \node[inner sep=0pt] (l4) at (-2.46, -1.6) {\tiny Layer 4, Head 5};
    \node[inner sep=0pt] (l4) at (0, -1.6) {\tiny Layer 4, Head 8};
    \node[inner sep=0pt] (l4) at (2.46, -1.6) {\tiny Layer 6, Head 5};
    \node[inner sep=0pt] (l4) at (4.92, -1.6) {\tiny Layer 8, Head 2};
    \end{tikzpicture}
    \caption{Comparison of different upper bounds on the spectral norm of the softmax Jacobian. Our bound ${g_{1}(P^h_{i, :})}/{\|\smjacmat{P^h_{i, :}}\|_2}$ is blue,  Gershgorin $h(P^h_{i, :})/{\|\smjacmat{P^h_{i, :}}\|_2}$ is orange and interlacing $\max_{j}P^h_{i,j}$ is green. As anticipated, our bound is tighter than the others. At the same time, the interlacing theorem bound may have large relative approximation error in comparison with the other bound, which is supported by \Cref{tab:ratio_comparison}. Heads, query pixels and ImageNet samples are kept the same for \Cref{fig:attention_maps}.}
    \label{fig:upper_bound_hist}
\end{figure}

This example indicates that bounds obtainable by eigenvalue interlacing theorem can be improved by utilizing the structure of $\smjacmat{x}$.
In \citep{nair2025softmax} authors utilize the structure of $\smjacmat{x}$ and prove that $h(x) \coloneqq 2x_{(1)}(1 - x_{(1)})$ upper bounds the spectral norm of the softmax Jacobian using Gershgorin theorem:
\begin{equation}\label{eq:gershgorin_bound}
    \left\|\mathcal{J}_{\texttt{sm}}(z)\right\|_2 \leqslant 2 \max_{i=1,\dotsc,N}\left(\texttt{sm}(z)_i\right)(1 - \max_{i=1,\dotsc,N} \texttt{sm}(z)_i).
\end{equation}
We argue that the structure of $\smjacmat{x}$ can be utilized even better than it is done in \citep{nair2025softmax}. Below we formulate our main theoretical result about softmax Jacobian singular values.
\begin{definition}\label{def:g}
    Let $x\in\mathbb{R}^N_{\geqslant 0}$ and let $x_{(k)}$ be the $k$-th largest component of $x$. For $k = 1, \dots, N-1$ define
    \begin{equation}\label{def:g_func}
        g_k(x) \coloneqq x_{(k)}(1 - x_{(k)} + x_{(k+1)})
    \end{equation}
    and, for consistency, $g_{N}(x) \equiv 0$.
\end{definition}

\begin{theorem}\label{th:sm_jac_bound}
    Let $x \in \mathbb{R}^N_{\geqslant 0}$ satisfy $\mathbf{1}^{\top} x = 1$, where $\mathbf{1}$ is a vector of all ones.
    Then, the singular values $\sigma_i(A)$ of $A = \mathrm{diag}(x) - xx^\top$ admit the following interlacing property with $x_{(i)}$ and $g_{i}(x)$ from \Cref{def:g}:
    \begin{equation}
        x_{(1)} \geqslant g_1(x) \geqslant \sigma_{1}(A)  \geqslant
        x_{(2)}  \geqslant g_2(x) \geqslant \sigma_{2}(A) \geqslant
        \dots \geqslant x_{(N)} \geqslant g_{N}(x) \geqslant \sigma_N(A).
    \end{equation}
\end{theorem}
\begin{proof}
    See Appendix~\ref{sec:proof_sm_jac_bound} for the proof.
\end{proof}

\begin{remark}
    From the theorem above it immediately follows that softmax is $1/2$-Lipschitz \wrt Euclidean norm: $\, \forall x, y \in \mathbb{R}^{n}$ we have
    \begin{equation}
         \|\texttt{sm}(x) - \texttt{sm}(y)\|_2 \leqslant 1/2 \cdot \|x-y\|_2.
    \end{equation}
\end{remark}

\subsection{How accurate is the new upper bound?}
We highlight that \Cref{th:sm_jac_bound} for the first singular value is tighter than direct application of both Gershgorin ($h(x)$) and interlacing theorems ($x_{(1)}$):
\begin{equation}
    \sigma_1(A) \leqslant g_{1}(x) \leqslant h(x), \quad
    \sigma_1(A) \leqslant g_{1}(x) \leqslant x_{(1)}.
\end{equation}
\begin{wraptable}{r}{0.5\textwidth}  
    \begin{minipage}{\linewidth}
  \centering
  \caption{Mean relative errors of different softmax Jacobian spectral norm estimation methods. Values are averaged across all heads, layers and samples.}
  \setlength{\tabcolsep}{8pt}
  \label{tab:ratio_comparison}
  \begin{tabular}{ccc}
    \hline
    Th.~\ref{th:sm_jac_bound} (ours) & Th.~\ref{th:interlacing} & Eq.~\ref{eq:gershgorin_bound} \\
    $0.035$ & $20.19$ & $0.85$ \\
    \hline
  \end{tabular}
  \end{minipage}
\end{wraptable}

\begin{wrapfigure}{r}{0.5\textwidth}
    \begin{minipage}{\linewidth}
    \centering
    \includegraphics[width=0.6\linewidth]{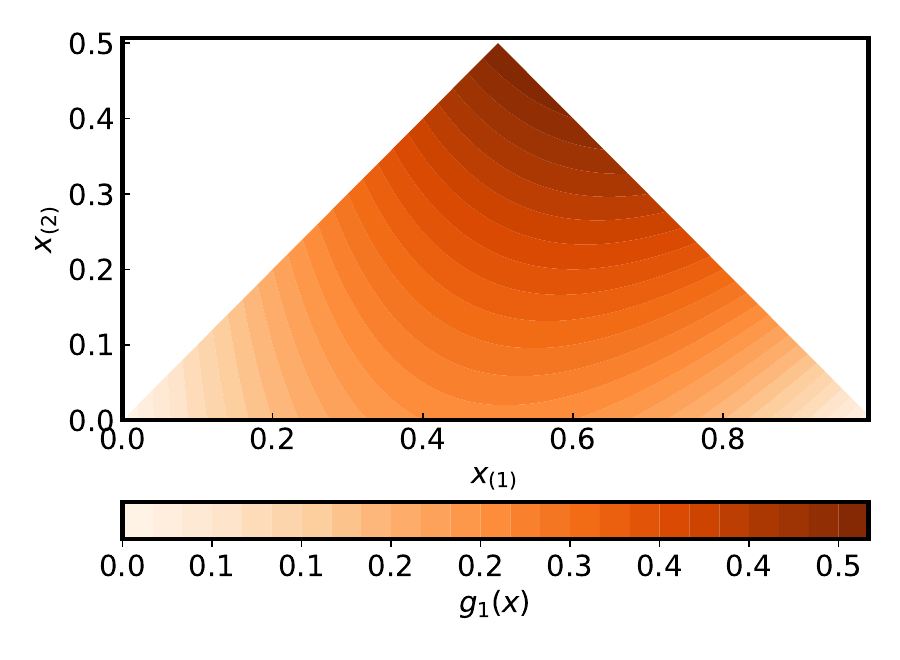}
    \caption{$g_1(x)$ for $x \in \mathbb{R}^{n}: \, \mathbf{1}^{\top}x=1$. $\max_{x} g_1 = 1/2$ is attained at $(x_{(1)}, x_{(2)})=(1/2,1/2)$ and $\min_{x} g_1 =0$ is attained at $(1,0)$, $(0, 0)$.}
    \label{fig:g1-region} 
    \end{minipage}
\end{wrapfigure}

Now let us examine what happens on a real-world example. 
\Cref{fig:attention_maps} presents various ViT-L attention maps for a certain sequence token $i$, averaged over 1000 ImageNet samples.
We report the average exact value of $\|\smjacmat{P^h_{i, :}}\|_2$ and the corresponding upper bounds derived via our approach and the Gershgorin theorem.

We observe that our bound consistently outperforms the Gershgorin theorem bound in all the examples.
A more detailed view on the distribution of the upper bounds is given in \Cref{fig:upper_bound_hist}.
In this figure, we plot the distribution of the ratio $g_1(P^h_{i, :}) / \|\smjacmat{P^h_{i, :}}\|_2$ and compare it with the distribution of the ratio $h(P^h_{i, :})/\|\smjacmat{P^h_{i, :}}\|_2$ for the same setting as on \Cref{fig:attention_maps}.
It can be seen that our bound is noticeably tighter than the other bounds.
To additionally show that our bound is sharper than the bound from eigenvalue interlacing theorem, we provide \Cref{tab:ratio_comparison} which reports mean relative approximation error for different softmax Jacobian spectral norm upper bounding methods.
Similarly to setups for \Cref{fig:attention_maps} and \Cref{fig:upper_bound_hist}, we average the relative approximation error across all token distributions, heads and layers.

Another noticeable effect is that the smaller the norm values, the closer the distribution becomes to either categorical or uniform distributions.
To understand this, we draw the behavior of $g_1(x)$ in \Cref{fig:g1-region}.
We see that these components can be small only at the corners $(1, 0)$, $(0, 0)$.

\begin{figure}[h]
    \centering
    \scalebox{0.9}{
    \begin{tikzpicture}
    \node (picture) at (0, 0) {\includegraphics[width=1\linewidth]{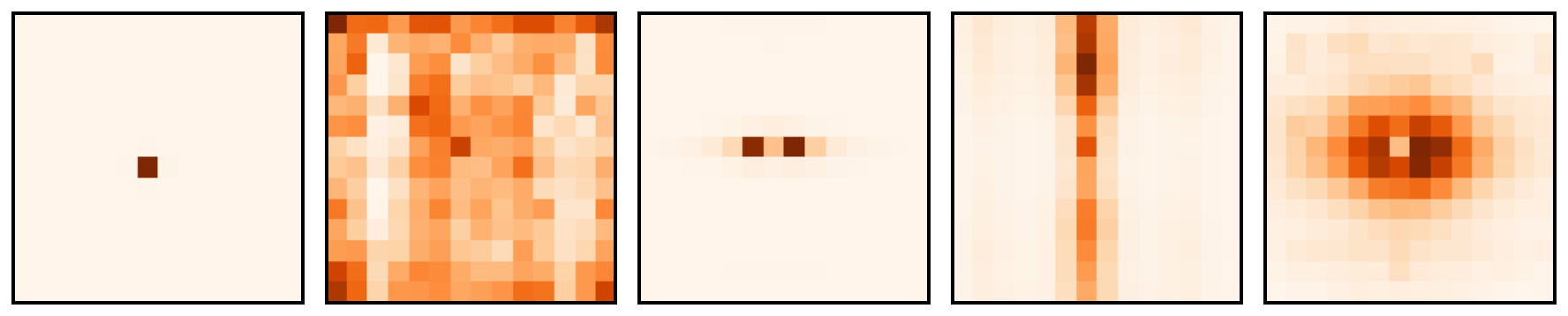}};
    
    \node[inner sep=0pt] (l1) at (-4.85, 1.8) {\tiny exact: 0.050};
    \node[inner sep=0pt] (l2) at (-4.85, 1.6) {\tiny our bound: 0.053};
    \node[inner sep=0pt] (l3) at (-4.85, 1.4) {\tiny gershgorin: 0.068};
    \node[inner sep=0pt] (l4) at (-4.85, -1.4) {\tiny Layer 3, Head 4};

    \node[inner sep=0pt] (l1) at (-2.45, 1.8) {\tiny exact: 0.021};
    \node[inner sep=0pt] (l2) at (-2.45, 1.6) {\tiny our bound: 0.021};
    \node[inner sep=0pt] (l3) at (-2.45, 1.4) {\tiny gershgorin: 0.042};
    \node[inner sep=0pt] (l4) at (-2.45, -1.4) {\tiny Layer 4, Head 5};

    \node[inner sep=0pt] (l1) at (0, 1.8) {\tiny exact: 0.253};
    \node[inner sep=0pt] (l2) at (0, 1.6) {\tiny our bound: 0.271};
    \node[inner sep=0pt] (l3) at (0, 1.4) {\tiny gershgorin: 0.412};
    \node[inner sep=0pt] (l4) at (0, -1.4) {\tiny Layer 4, Head 8};

    \node[inner sep=0pt] (l1) at (2.45, 1.8) {\tiny exact: 0.050};
    \node[inner sep=0pt] (l2) at (2.45, 1.6) {\tiny our bound: 0.082};
    \node[inner sep=0pt] (l3) at (2.45, 1.4) {\tiny gershgorin: 0.084};
    \node[inner sep=0pt] (l4) at (2.45, -1.4) {\tiny Layer 6, Head 5};

    \node[inner sep=0pt] (l1) at (4.85, 1.8) {\tiny exact: 0.071};
    \node[inner sep=0pt] (l2) at (4.85, 1.6) {\tiny our bound: 0.074};
    \node[inner sep=0pt] (l3) at (4.85, 1.4) {\tiny gershgorin: 0.137};
    \node[inner sep=0pt] (l4) at (4.85, -1.4) {\tiny Layer 8, Head 2};
    \end{tikzpicture}
    }
    \caption{Attention probabilities for ViT-L
    model for certain heads. The title of each heatmap describes true average value of $\|\smjacmat{P^h_{i, :}}\|_2$ together with its our and Gershgorin bounds. Each attention map is averaged across $1000$ ImageNet samples.}
    \label{fig:attention_maps}
\end{figure}

\section{A new upper bound for self-attention Jacobian}
\label{sec:refined_attn_bound}
Previous works that aimed to bound the self-attention Jacobian spectral norm \citep{hu2024specformer, castin2023smooth, dasoulas2021lipschitz} were primarily focused on the dependency of the self-attention Jacobian on the norms of inputs and model weights. Such an approach is rather limiting as it completely disregards all the information about attention map scores, which is a core non-linear mechanism inside self-attention. In the prior paper \citep{hu2024specformer} authors omit the term containing softmax by bounding corresponding spectral norm by~$1$. In the paper \citep{castin2023smooth} authors use probabilistic techniques to bound the spectral norm of the self-attention Jacobian, but still omit any softmax scores' dependencies in their main results.
In this section, we introduce a novel approach to bound the spectral norm of self-attention Jacobian taking into account attention map scores nature. 
Below we provide our core theoretical contribution.

\begin{theorem}\label{th:attn_jac_bound}
    For the self-attention mechanism defined in \Cref{eq:attention},
    the following inequality holds:
    \begin{equation} \label{eq:main_bound}
        \|\mathcal{J}_{\mathtt{Attn}_h}(X)\|_2 \leqslant \left\|{W_h^V}\right\|_{{2}}\cdot 
        \Big(\left\|{P^h}\right\|_{{2}} + 
        2 \left\|{X}\right\|_{{2}}^2 \left\|{A_h}\right\|_{{2}}\max_{i}\left\|\smjacmat{P^h_{i, :}}\right\|_{{2}}\Big),
    \end{equation}
    where $\smjacmat{P_{i, :}^h}$ is defined as in \Cref{eq:smjacmat}.
\end{theorem}

\begin{proof}
    See Appendix~\ref{sec:proof_attn_jac_bound} for the proof.    
\end{proof}
\begin{remark}
Note that due to the continuity of the softmax function, we can readily obtain the local Lipschitz bound in a neighborhood of $X$ from \Cref{eq:main_bound}.
\end{remark}
 
\begin{remark}
Since multi-head self-attention is simply a linear combination of individual attention heads followed by an output projection, we can bound its norm by the sum of the corresponding spectral norms of each head's transformation:
\begin{equation}\label{eq:multihead_bound}
\begin{split}
    & \left\|
    \begin{bmatrix}
            \mathcal{J}_{\texttt{Attn}_1}(X) & \dots & \mathcal{J}_{\texttt{Attn}_H}(X)
    \end{bmatrix}
    \right\|_2 
    \leqslant \sum_{i=1}^{H}\left\|\mathcal{J}_{\texttt{Attn}_i}(X)\right\|_2.
\end{split}
\end{equation}
\end{remark}

This justifies our use of a summed penalty across heads and layers in the regularization term in \Cref{sec:lip_const_control}.

Our result provides a conditional interpretation of sink-like behavior noticed in papers \citep{gu2025attentionsinkemergeslanguage, xiao2023efficient} through the lens of self-attention Jacobian matrix spectral properties. The Lipschitz constant decreases when row-wise attention distributions enter a highly concentrated, near-categorical regime. Importantly, the similarity to attention sinks is that both involve concentration of attention mass, though it is not identical to the usual empirical notion of attention sinks, where the same token or position consistently attracts attention across all inputs. Overall, our upper bound provides a partial theoretical explanation that sufficiently concentrated attention mass produces in a stabilizing effect.

\subsection{Comparison with existing bounds}
\label{sec:comparixon_with_existing_bounds}
Below we recall existing results for upper bounds of dot-product self-attention local Lipschitz constant in order to compare them with our bound. 

Let $B_R(y) \subset \mathbb{R}^D$ represent an open $\ell_2$-ball of radius $R$  centered at $y \in \mathbb{R}^{D}$. Additionally, let $B_R^N(y) \subset \mathbb{R}^{N \times D}$ for $y \in \mathbb{R}^{D}$ denote a set of matrices $Y \in \mathbb{R}^{N \times D}$ such that $\forall \, i=1, \dotsc, N:\,\|Y_{i, :}\|_2 \in B_R(y)$ and $\mathrm{vec}(W)$ is a column-wise vectorization of a matrix into a vector.

\begin{theorem}[\citep{hu2024specformer}]\label{th:specformer}
    The spectral norm of a single-head self-attention admits the following bound:
\begin{equation}
    \left\|\mathcal{J}_{\mathtt{Attn}_h}(X)\right\|_2 \leqslant N(N+1) \cdot \|X\|_F^2 \cdot \left(\|W^V_h\|_2\|W^Q_h\|_2\|W^K_h\|_2 + \|W^V_h\|_2\right).
\end{equation}
\end{theorem}
The upper bound proposed in the paper \citep{hu2024specformer} is far from the real spectral norm because it depends quadratically on $N$.
In paper \citep{castin2023smooth} the authors provide a tight bound for the local Lipschitz constant of unmasked self-attention and propose to bound the spectral norm of self-attention Jacobian  as $\mathcal{O}(\sqrt{N})$ in $N$.

\begin{theorem}[\citep{castin2023smooth}]\label{th:how_smooth_is_attention}
Let $X \in \mathbb{R}^{N \times D}$ and let $R>0$ be such that $X \in B^N_R(0) \subset \mathbb{R}^{N \times D}$. Then, 
\begin{equation}
    \left\|\mathcal{J}_{\mathtt{Attn}_h}(X)\right\|_2 \leqslant \sqrt{3}\|W_h^V\|_2\left(\|A_h\|_2^2R^4(4N+1) + N\right)^{\frac{1}{2}}.
\end{equation}
\end{theorem}
One may notice that the bound provided in \Cref{th:attn_jac_bound} has a worse order in $N$ than the bound provided in \Cref{th:how_smooth_is_attention}. Indeed, the proposed bound can be linear in $N$ provided that $X$ has the spectral norm of order $\mathcal{O}(\sqrt{\max(N, D)})$ for both random or Pre-LayerNorm of $X$ \citep{xiong2020preln}.  However, in our upper bound we have a term $\max_{i=1,\dotsc,N}\|\smjacmat{P_{i, :}^h}\|_2$ which in practice makes our bound more precise than the one provided in \Cref{th:how_smooth_is_attention}. To validate it empirically, we take ViT-L and process ImageNet samples showing that our bound on the self-attention Jacobian spectral norm is tighter than prior art for a sequence length $N = 196$ across all heads and layers (see \Cref{fig:vit_l_consts}). We observe that our bound from \Cref{th:attn_jac_bound} is tighter, outperforming all existing bounds within a considerable margin. This observation supports our theory regarding the influence of attention map distribution on self-attention local Lipschitzness.

Despite the fact that in practice $\max_{i=1,\dotsc,N}\|\smjacmat{P_{i, :}^h}\|_2$ gives more advantage than $\sqrt{N}$, we are also able to improve the bound from \Cref{th:how_smooth_is_attention} by a constant factor by slightly altering the end of our proof 
(see Appendix~\ref{sec:proofattn_jac_bound_without_sm_jac}).

\begin{prop}\label{prop:attn_jac_bound_without_sm_jac}
Under the assumptions of \Cref{th:how_smooth_is_attention}:
    \begin{equation}
        \left\|\mathcal{J}_{\mathtt{Attn}_h}(X)\right\|_2 \leqslant \|W_h^V\|_2 \bigl(\sqrt{N} + 2\sqrt{N}R^2\|A_h\|_2 \bigr).
    \end{equation}
\end{prop}

\begin{figure}[t]
    \centering
    \resizebox{0.8\textwidth}{!}{
    \begin{tikzpicture}
    \node (picture) at (0, 0) {\includegraphics[width=1\linewidth]{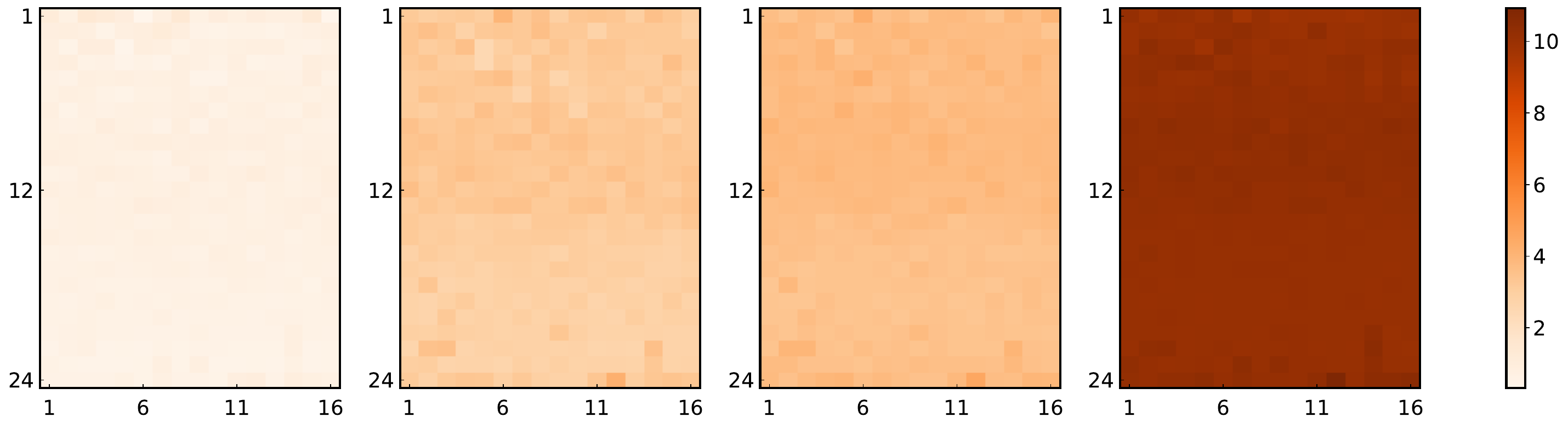}};
    \node[inner sep=0pt] (l1) at (-4.6, 2) {\tiny Exact};
    \node[inner sep=0pt] (l2) at (-4.6, 1.8) {\tiny Mean: 4.8e1};
    \node[inner sep=0pt] (l1) at (-1.8, 2) {\tiny \Cref{th:attn_jac_bound}};
    \node[inner sep=0pt] (l2) at (-1.8, 1.8) {\tiny Mean: 1.3e3}; 
    \node[inner sep=0pt] (l1) at (1, 2) {\tiny\Cref{th:how_smooth_is_attention}};
    \node[inner sep=0pt] (l2) at (1, 1.8) {\tiny Mean: 4.7e3}; 
    \node[inner sep=0pt] (l1) at (3.8, 2) {\tiny \Cref{th:specformer}}; 
    \node[inner sep=0pt] (l2) at (3.8, 1.8) {\tiny Mean: 1.3e10}; 
    \node[inner sep=0pt] (h1) at (-4.6, -1.8) {\tiny Head};
    \node[inner sep=0pt] (h1) at (-1.8, -1.8) {\tiny Head};
    \node[inner sep=0pt] (h1) at (1, -1.8) {\tiny Head};
    \node[inner sep=0pt] (h1) at (3.8, -1.8) {\tiny Head};
    \node[inner sep=0pt, rotate=90] (l3) at (-6.25, 0) {\tiny Layer};
    \end{tikzpicture}
    }
    \caption{Self-attention Jacobian spectral norms' comparison for each head and layer of the ViT-L model across 5000 ImageNet samples. The title of each heatmap represents the mean value of the bound across all heads and layers.}
    \label{fig:vit_l_consts}
\end{figure}

\section{Gradient sensitivity to attention map distributions}
\label{sec:gradient_sensitivity}

While the $1/2$-Lipschitzness of softmax already motivates increased stability of transformers to exploding gradients due to the upper bound on each layer's sensitivity, our approach allows for a more fine-grained characterization of gradient dynamics as well.

Although bounding the local Lipschitzness was noticed to have a positive effect on the robustness properties of deep learning models \citep{yang2020accVSrob}, it was also shown that it may lead to the well-known accuracy-robustness trade-off \citep{bethune2022payattention2loss}. 
Unlike previous works, that have mostly focused on generalization properties of locally Lipschitz functions \citep{muthukumar23sparsityaware}, we aim to show that softmax distributions significantly impact gradient dynamics and affect optimization --- observations, that have independently been made in studies, which focused on signal propagation, training dynamics and ``sink''-like effects in large-scale transformer-based models \citep{noci2022signalpropagation, zhai2023stabilizing, gu2025attentionsinkemergeslanguage}. 

To illustrate the point, notice that the gradient of an arbitrary parameter of an attention block $X$ of the shallower layer can be expressed as a matrix-vector multiplication of the gradient w.r.t. its output  with the adjoint Jacobian:
\begin{equation}
    \left \langle \frac{\partial L}{\partial \texttt{Attn}_h(X)}, \mathcal{J}_{\texttt{Attn}_h}(X) dX \right \rangle = \left \langle \mathcal{J}_{\texttt{Attn}_h}(X)^\top \frac{\partial L}{\partial \texttt{Attn}_h(X)}, dX \right \rangle.
\end{equation}
Then, with a little abuse of notation, we can bound gradient norms by the product of the Jacobian spectral norm and a gradient obtained from backpropagation. 

Thus, for any $W \in \{W_h^Q, W_h^K, W_h^V\}$ from self-attention we have:
\begin{gather}\label{eq:jac_grad_relation}
    \left \| \frac{\partial L}{\partial W} \right \|_F = \left \| \mathrm{vec}\left(\cfrac{\partial L}{\partial W}\right) \right \|_2 \leqslant \left \| \cfrac{\partial \texttt{Attn}_h(X)}{\partial W}\right \|_2 \left \| \frac{\partial L}{\partial \texttt{Attn}_h(X)} \right \|_F.
\end{gather}

It can be seen that gradients \wrt the weights depend not only on the sensitivity of deeper layers to attention block's output, but on the spectral norm of the gradient \wrt $W_h^Q, W_h^K, W_h^V$. 
To show a connection between the attention map distribution nature and the gradient norm, below we provide a distribution-dependent upper bound for $\|\partial \texttt{Attn}_h(X)/\partial W\|_2$, where $W \in \{W_h^Q, W_h^K, W_h^V\}$.

\begin{prop}\label{prop:gradient_norms_from_attention_jac}
    For the self-attention mechanism defined in \Cref{eq:attention}, gradient norms can be bounded as follows:
\begin{gather}
    \left \| \frac{\partial \mathtt{Attn}_h(X)}{\partial W_h^Q} \right \|_2 \leqslant \cfrac{1}{\sqrt{d}} \lVert W_h^V \rVert_2 \lVert W_h^K \rVert_2 \lVert X \rVert_2^3 \max_{i}\left\|\smjacmat{P^h_{i, :}}\right\|_{{2}} \\
    \left \| \cfrac{\partial \mathtt{Attn}_h(X)}{\partial W_h^K} \right \|_2 \leqslant \cfrac{1}{\sqrt{d}} \lVert W_h^V \rVert_2 \lVert W_h^Q \rVert_2 \lVert X \rVert_2^3 \max_{i}\left\|\smjacmat{P^h_{i, :}}\right\|_{{2}} 
\end{gather}
and for $W_h^V$: $\displaystyle \left \| \cfrac{\partial \mathtt{Attn}_h(X)}{\partial W_h^V} \right \|_2 \leqslant \lVert P \rVert_2 \lVert X \rVert_2 \leqslant \sqrt{\max_j \textstyle{\sum_{i=1}^N} P_{ij}} \ \lVert X \rVert_2$.
\end{prop}

\begin{proof}
    See Appendix~\ref{sec:proof_gradient_norms_from_attention_jac} for the proof.    
\end{proof}

To support Proposition~\ref{prop:gradient_norms_from_attention_jac} empirically, we measure gradient spectral norm \wrt $Q$ and $K$ matrices for 1000 ImageNet test samples along with $g_1(P^h_{i, :})$ ---  the upper bound on $\max_{i}\|\smjacmat{P_{i, :}^h}\|_2$. 
~\Cref{fig:grad_norm_to_g_1_corr} shows the correlation between gradient norms divided by $\|X\|_2^3$, which allows to exclude input norm dependency in our bound (following Proposition~\ref{prop:gradient_norms_from_attention_jac}), and $\max_i g_1(P^h_{i, :})$.  

Let us make a few observations regarding the gradient structure. The first thing to be noted is the difference in gradient dynamics for $W_h^Q, W_h^K$ and $W_h^V$. As $P$ is row-stochastic and LayerNorm is applied to $X$,  we can upper bound $\lVert PX \rVert_2$ by $\sqrt{ND}$ and notice that $\partial \texttt{Attn}_h(X)/{\partial W_h^V}$ does not necessarily lead to gradient vanishing. However, we could observe that uniform attention distribution pushes $\lVert P \rVert_2$ towards $1$, unlike categorical attention patterns, where the bound approaches its maximal value $\sqrt{N}$. 

In contrast, the gradients w.r.t. the $W_h^Q$ and $W_h^K$ may fade, shrinking towards $0$ in the proximity of \emph{either uniform or categorical distributions for all tokens in the sequence}, an effect overlooked by \citep{noci2022signalpropagation}. It should also be noted that although both lead to the gradient vanishing, the uniform case attains its minimum at $1/N$, scaling with sequence length, while the categorical one shrinks right to $0$.

Similar effects that unify observed regimes were identified in prior literature, but haven't been attributed to the spectral properties of attention maps. When attention distributions inside one attention map all become uniform-like, an update of the attention block becomes close to a rank-$1$ matrix, leading to the rank collapse phenomenon \citep{dong2021attention} and training instability \citep{chen2025condensation, noci2022signalpropagation}. On the other hand, attention spikes have been attributed to ``sink-like'' effects and entropy collapse, often phrased as an ability to perform ``no-op'' operation in deeper layers \citep{barbero2025llms}. Similarly to the uniform case, it has been shown to induce training oscillations or divergence \citep{zhai2023stabilizing}, with recent ``sinkless'' architectures displaying superior results on most benchmarks \citep{qiu2025gated}. 
Importantly, our bound provides a precise and efficient measure of proximity to these regimes, serving as a practical proxy for vanishing gradients during the forward pass.

\begin{figure}[t]
    \centering
    \includegraphics[width=0.7\linewidth]{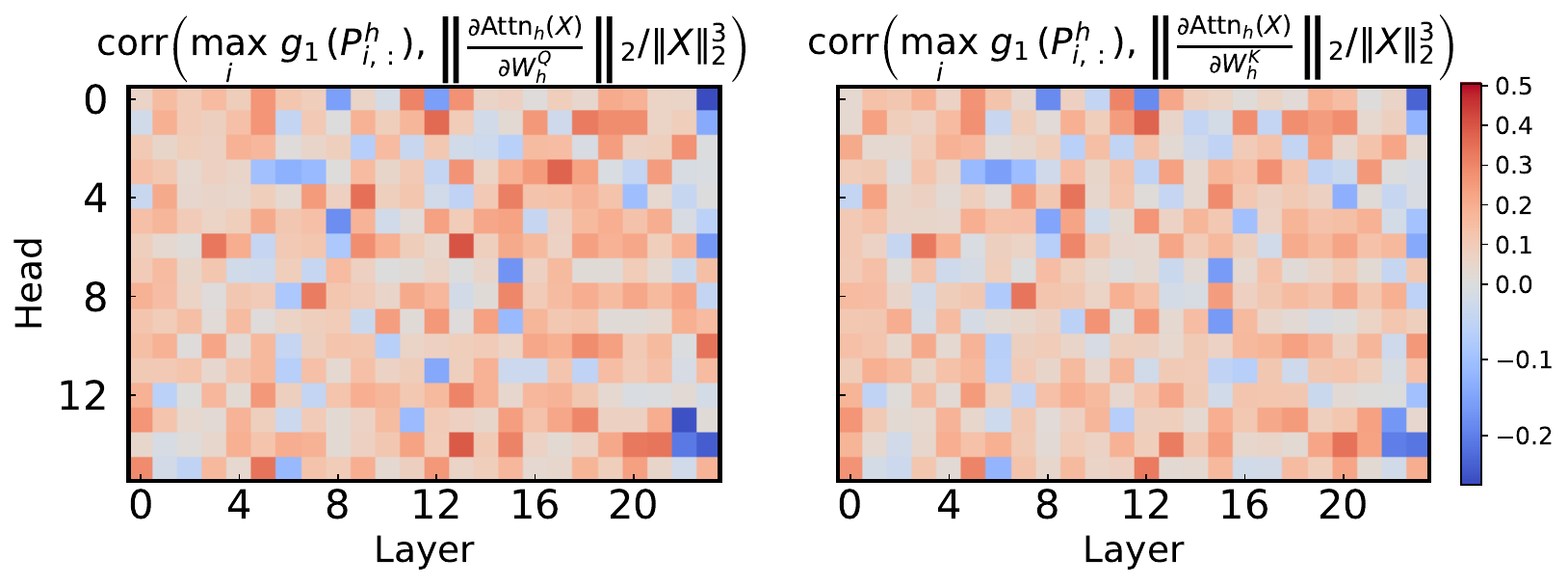}
    \caption{Correlation between norms for $W_Q$, $W_K$ and $\max_i g_1(P_{i,:}^h)$ averaged across 1000 ImageNet samples for ViT-L model. To discard input norm dependency, we divide gradient norm by $\|X\|_2^3$, following the bound proposed in Proposition~\ref{prop:gradient_norms_from_attention_jac}.}
    \label{fig:grad_norm_to_g_1_corr}
\end{figure}

\section{\texttt{JaSMin} regularizer}
\label{sec:lip_const_control}

From \Cref{th:attn_jac_bound} we observe that the spectral norm of the self-attention depends on three types of components: (1) input-only part $\|X\|_2^2$, (2) parts depending on weight matrices (\eg $\|A_h\|_2$ and $\|W_h^V\|_2$) and (3) mixed terms that are functions of the matrix $P^h$, where $P^h$ in turn depends on both weights and inputs ($\|P^h\|_2$ and $\|\max_{i}\smjacmat{P^h_{i, :}}\|_2$). 
The regularizer that we present in this section as a method to improve robustness and validate our findings targets the latter terms. 
In particular, we propose $2$ strategies for the choice of a penalty:
\begin{equation}
    \mathcal{L} = \mathcal{L}_{\texttt{model}} + \lambda \mathcal{L}_{\texttt{JaSMin}_{k}},
\end{equation}
where $\mathcal{L}_{\texttt{model}}$ is a model loss function, $\lambda$ is the regularization parameter, and 
 $\texttt{JaSMin}_{k}$ is the proposed modification.
The first way is to penalize $\log(\|\smjacmat{P_{i, :}^h}\|_2)$ as it naturally constrains the upper bound from \Cref{th:attn_jac_bound}:
\begin{equation} \label{eq:loss_reg_1}
    \mathcal{L}_{\texttt{JaSMin}_{k=0}} = \sum_{l, h=1}^{L, H} \max_{i} \log\left(g_1(P_{i, :}^{l, h})\right).
\end{equation}
 
The second strategy is based on the interlacing property of the softmax Jacobian singular values: we aim to penalize the logarithm of the following ratio:
\begin{equation} \label{eq:loss_reg_2}
    \mathcal{L}_{\texttt{JaSMin}_k} =  \sum_{l, h=1}^{L, H} \max_{i} \log\left(\frac{g_1(P_{i, :}^{l, h})}{g_k(P_{i, :}^{l, h})}\right), \quad k > 1.
\end{equation}

\begin{table*}[htbp]
  \centering
  \caption{Comparison of \texttt{Specformer} and \texttt{JaSMin} approaches for ViT-B on CIFAR-100. Number after the name of attack denotes attack budget. We evaluate PGD and AutoAttack with $4$ and $10$ steps respectively. Each run is averaged across $3$ seeds.}
  \resizebox{\textwidth}{!}{
  \begin{tabular}{lcccccc}
    \toprule
    \multicolumn{1}{c}{Method} & \multicolumn{6}{c}{CIFAR-100} \\
    \cmidrule(lr){2-7} 
    & Standard & FGSM2 & FGSM4 & PGD2 & PGD4 & AA2 \\
    \midrule 
    \texttt{Baseline} & $91.05_{\pm 0.02}$ & $46.64_{\pm 0.71}$ & $39.86_{\pm 0.60}$ & $27.47_{\pm 0.40}$ & $13.11_{\pm 0.33}$ & $2.10_{\pm 0.01}$ \\
    \texttt{Specformer}$_{(10^{-2},10^{-2},10^{-2})}$ & $90.06_{\pm 0.01}$ & $44.90_{\pm 1.86}$ & $37.65_{\pm 1.20}$ & $28.08_{\pm 0.96}$ & $12.61_{\pm 0.30}$ & $\underline{3.05}_{\pm 0.02}$ \\
    \texttt{Specformer}$_{(0,0,10^{-2})}$ & $\mathbf{91.14}_{\pm 0.05}$ & $46.58_{\pm 0.04}$ & $39.97_{\pm 0.07}$ & $27.59_{\pm 0.01}$ & $12.56_{\pm 0.12}$ & $2.10_{\pm 0.01}$ \\
    \texttt{JaSMin}$_{k=0, \lambda=1e\!-\!2}$ & $88.83_{\pm 0.01}$ & $\underline{48.28}_{\pm 0.31}$ & $\underline{41.38}_{\pm 0.29}$ & $\mathbf{32.58}_{\pm 0.35}$ & $\mathbf{18.90}_{\pm 0.28}$ & $\mathbf{3.62}_{\pm 0.01}$ \\
    \texttt{JaSMin}$_{k=10, \lambda=1e\!-\!3}$ & $\underline{91.08}_{\pm 0.02}$ & $\mathbf{48.90}_{\pm 0.05}$ & $\mathbf{42.34}_{\pm 0.03}$ & $\underline{30.29}_{\pm 0.09}$ & $\underline{15.47}_{\pm 0.09}$ & $2.45_{\pm 0.01}$ \\
    \bottomrule
  \end{tabular}
  }
  \label{tab:results_cifar100}
\end{table*}

As we show below in Proposition~\ref{prop:uniform_on_k}, this regularizer relaxes the categorical distribution constraint and yields more flexibility towards a uniform distribution.

Note that the memory overhead of \texttt{JaSMin} is only $\mathcal{O}(1)$ due to the associative properties of the summation and maximum operations, while \texttt{Specformer} needs to store additional $\mathcal{O}(LD)$ memory for singular vectors at training stage.

\paragraph{\textbf{Regularization of the upper bound.}}
\label{par:reg_with_largest_singvalue}

By minimizing \Cref{eq:loss_reg_1} that is based on the upper bound on the first singular value, we implicitly constrain $\|{\smjacmat{P_{i, :}^{l, h}}}\|_{{2}}$  for every layer $l$ and head $h$. 
Let us precisely derive what we can say about $P_{i, :}^{l, h}$ if $g_1$ is bounded by a constant $\gamma < {1}/{4}$.
By solving the quadratic inequality, we arrive at:
\begin{equation}\label{eq:quadratic_inequality}
      (P^{l, h}_{i, :})_{(1)} \leqslant \frac{1 - \sqrt{1 - 4\gamma}}{2} \; \lor \; (P^{l, h}_{i, :})_{(1)} \geqslant \frac{1 + \sqrt{1 - 4\gamma}}{2}.
\end{equation}

Hence, with this regularization term we enforce attention distributions to become either ``more'' uniform or ``more'' categorical, excluding all the intermediate options, according to \Cref{fig:g1-region}.

\paragraph{\textbf{Regularization of the ratio.}}
The interlacing property of \Cref{th:sm_jac_bound} allows us to bound the ratio of singular values:
\begin{equation}\label{eq:ratio_bound}
\sigma_1(P_{i, :}^{l, h}) / \sigma_{k-1}(P_{i, :}^{l, h}) \leqslant
{g_1\left(P_{i, :}^{l, h}\right)}/{g_k\left(P_{i, :}^{l, h}\right)}.
\end{equation}
An interesting observation that we make is that by bounding the upper bound from \Cref{eq:ratio_bound}, we can also constrain the norm to be small.
The following proposition yields an upper bound on the spectral norm in this scenario.

\begin{prop}\label{prop:uniform_on_k}
Let   $1 \leqslant \gamma \leqslant {k}/{4}$ and
$
{g_1\left(P_{i, :}^{l, h}\right)}/{g_k\left(P_{i, :}^{l, h}\right)} \leqslant \gamma.
$
Then, we have:
$
\displaystyle{
\left\|{\smjacmat{P_{i, :}^{l, h}}}\right\|_{{2}} \leqslant \frac{1 - \sqrt{1 - {4\gamma}/{k}}}{2} = \mathcal{O}\left(\frac{\gamma}{k}\right).
}
$
\end{prop}
\begin{proof}
    See Appendix~\ref{sec:proof_of_prop_1} for the proof.    
\end{proof}

\begin{table*}[htbp]
  \centering
  \caption{Comparison of \texttt{Specformer} and \texttt{JaSMin} approaches for ViT-B on Imagenette. Number after the name of attack denotes attack budget. We evaluate PGD and AutoAttack with $4$ and $10$ steps respectively.}
  \begin{tabular}{lcccccc}
    \toprule
    \multicolumn{1}{c}{Method} & \multicolumn{6}{c}{Imagenette} \\
    \cmidrule(lr){2-7} 
    & Standard & FGSM2 & FGSM4 & PGD2 & PGD4 & AA2 \\
    \midrule 
    \texttt{Baseline} & $\underline{98.00}$ & $\mathbf{70.12}$ & $\underline{55.51}$	& $45.02$ & $15.13$	&6$.86$ \\
    \texttt{Specformer}$_{(10^{-2}, 0, 0)}$ & $\mathbf{98.12}$ & $\underline{70.03}$ & $\mathbf{55.95}$ & $45.76$	& $15.98$ & $7.31$  \\
    \texttt{JaSMin}$_{k=30, \lambda=10^{-2}}$ & $97.07$ & $68.67$ & $54.94$ & $\mathbf{48.18}$ & $\mathbf{20.67}$ & $\underline{10.62}$ \\
    \texttt{JaSMin}$_{k=70, \lambda=10^{-2}}$ & $96.37$ & $66.34$ & $51.66$ & $\underline{46.01}$ & $\underline{18.55}$ & $\mathbf{11.95}$ \\
    \bottomrule
  \end{tabular}
  \label{tab:results_imagenette}
\end{table*}

To understand what happens with the distribution, let us consider the corner case of $\gamma=1$.
In this case, due to the interlacing property from \Cref{th:sm_jac_bound}: 
\begin{equation}
x_{(2)}/x_{(k)} \leqslant g_1 / g_{k} = 1.
\end{equation}
One can show (see Appendix~\ref{sec:proof_of_prop_1})
that in this case $x_{(1)} = \dots = x_{(k)}$.
So any uniform distribution for top-$p$ elements, $k\leqslant p\leqslant N$ will fulfill this condition.
This yields more flexibility in the uniform distribution as opposed to using $\texttt{JaSMin}_{k=0}$.

It should be noted that our regularizer with $k > 1$ is not immediately applicable to models utilizing attention masking. For example, for causal attention mask, the regularizer will be unbounded on the first $(k-1)$ tokens for any $k > 1$.

\paragraph{\textbf{Method efficiency}.}

The \texttt{JaSMin} method introduces an additional computational cost of $\mathcal{O}\left(BLHN^2\right)$ FLOPs compared to \texttt{Specformer}'s $\mathcal{O}\left(LD^2\right)$ where $B$ is the batch size. Although \texttt{JaSMin}'s dependence on $N^2$ might appear to be a limitation, it does not pose a substantial computational cost in ViT experiments, as the sequence length is negligible for most CV models, which results in faster computations than  that of \texttt{Specformer}. As $N$ increases, the \texttt{JaSMin} overhead becomes progressively less significant, being overshadowed by the transformer's $\mathcal{O}\left(BL\left(ND^2+N^2D\right)\right)$ computational cost.

\section{Experiments}
\label{sec:experiments}

We train the ViT-B model \citep{dosovitskiy2020image} to classify images using CIFAR-10, CIFAR-100 \citep{krizhevsky2009learning} and Imagenette \citep{imagenette} datasets, and show that our regularizer is able to control the local Lipschitz constant, improving robust metrics without a significant quality drop (see \Cref{tab:results_cifar100}, \Cref{tab:results_imagenette}, \Cref{tab:ablation_results_cifar100} and \Cref{tab:ablation_results_cifar10} for more details). Our models are tuned on the respective datasets starting from checkpoints pretrained on the ImageNet dataset \citep{5206848}. Patch size is set to $4$ in this experiment for CIFAR-10 and CIFAR-100, while for Imagenette we utilize patch size equal to $16$.

For comparison we use the original ViT training setup with \texttt{Specformer} and \texttt{JaSMin} method. \texttt{Specformer} is parameterized with three regularization coefficients for the largest singular values of the self-attention query, key and value matrices. We reproduced the setup from \texttt{Specformer} paper \citep{hu2024specformer} and compare our models to theirs by substituting the regularization parameters only, the distribution prefix length $k$ and a regularization coefficient $\lambda$. 
To show that our method not only reduces the local Lipschitz constant, but also improves the robust metrics, we evaluate our models on several adversarial attacks. 

We report FGSM \citep{goodfellow2015explainingharnessingadversarialexamples} and PGD \citep{madry2019deeplearningmodelsresistant} with $4$ steps and AutoAttack (AA) \citep{pmlr-v119-croce20b} with $10$ steps. 
FGSM and PGD are configured with the attack budgets of $4/255$ and $2/255$, while AutoAttack, being more severe, is used with a smaller budget of $2/255$ to provide more evident results. All attack budgets are measured in terms of the infinity norm.
Other training details can be found in 
Appendix~\ref{sec:add_experiments}.
\begin{wrapfigure}{r}{0.5\textwidth}
    \centering
    \resizebox{0.47\textwidth}{!}{
    \begin{tikzpicture}
    \node (picture) at (0, 0) {\includegraphics[width=0.5\linewidth]{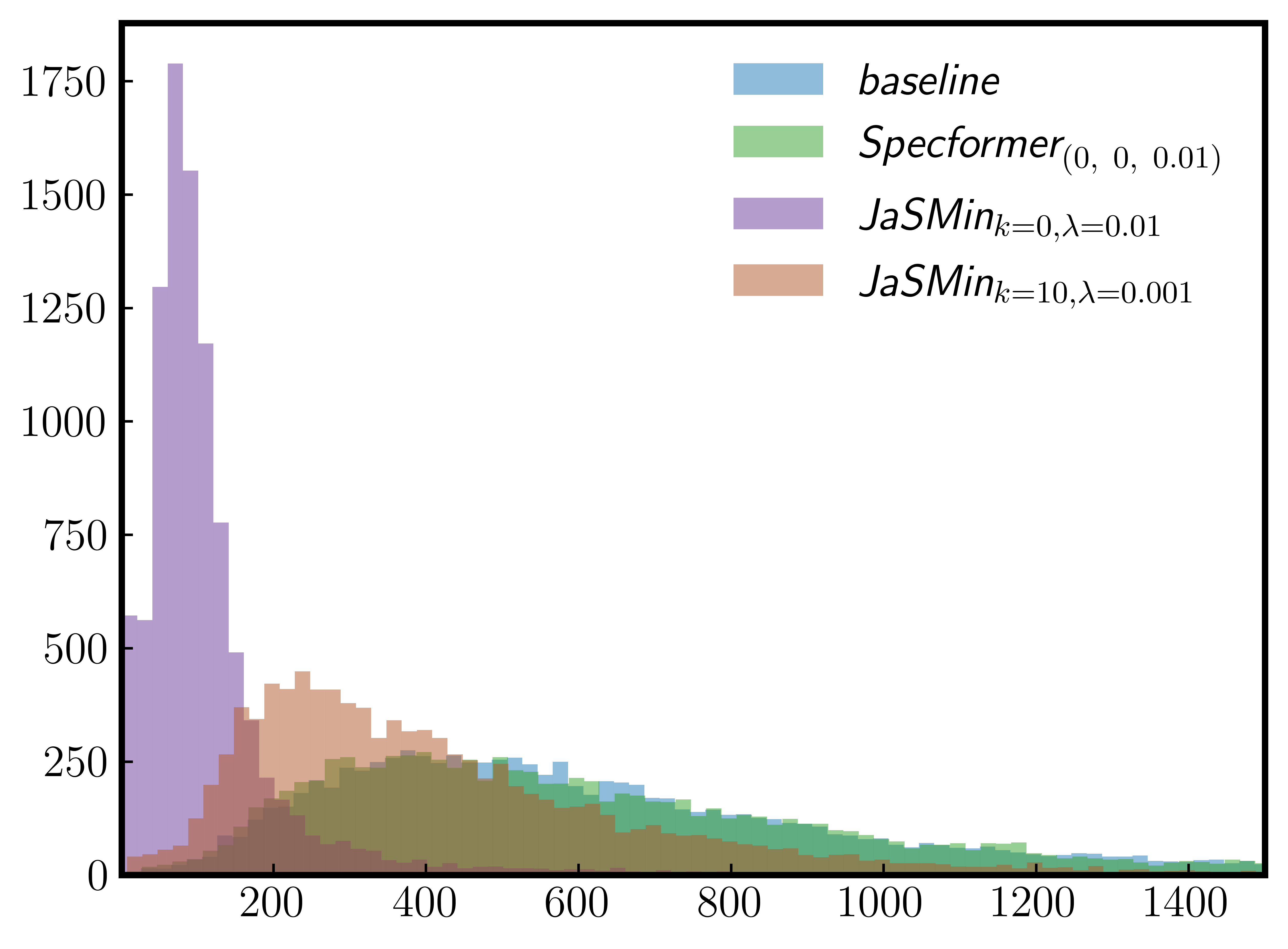}};
    \end{tikzpicture}
    }
    \caption{
    Comparison of distributions of the spectral norm of the whole model's Jacobian (trained on CIFAR-100) using different regularization methods. The spectral norm is computed via power method at each point of the test dataset.
    }
    \label{fig:model_Jacobian} 
\end{wrapfigure}
Additionally, we report gradient norm dynamics for different model weights in
Appendix~\ref{sec:gradient_dynamic_analyses} for better
understanding of the potential gradient vanishing problem.
Apart from metrics on adversarial attacks, we report the distribution of the local Lipschitz constant. We remove patch embedding layer from the model and compute Jacobian norm for the rest of the model, as \texttt{JaSMin} does not regularize positional embeddings and a convolutional layer. 
As illustrated in \Cref{fig:model_Jacobian}, our regularizer considerably reduces the local Lipschitz constant of the model, supporting our findings, while \texttt{Specformer} regularizer does not considerably change the mean value of the local Lipschitz constant.

\section{Conclusion}
We introduce a new theoretical upper bound for the spectral norm of the self-attention Jacobian. This refinement is based on the improved upper bound for the largest singular value of the softmax Jacobian. Through experiments with vision transformers we observe that the nature of the attention map distribution crucially affects the self-attention local Lipschitz constant. Additionally, we provide refined interlacing bounds for the softmax Jacobian singular values, provably improving the ones obtainable via an eigenvalue interlacing theorem. To highlight the influence of the attention map distribution on the self-attention local Lipschitz constant, we introduce a lightweight regularizer \texttt{JaSMin}. Apart from that, we show that controlling attention map distributions indeed results in a significant drop of the network's local Lipschitz constant.

\section*{Acknowledgements}
The work was supported by the grant for research centers in the field of AI provided by the Ministry of Economic Development of the Russian Federation in accordance with the agreement 000000C313925P4E0002 and the agreement with HSE University № 139-15-2025-009.
This research was supported in part through computational resources of HPC facilities at HSE University \cite{kostenetskiy2021hpc}.

\bibliographystyle{splncs04}
\bibliography{main}

\onecolumn

\appendix
\renewcommand{\theHsection}{appendix.\Alph{section}}

\numberwithin{figure}{section}
\numberwithin{table}{section}

\section{Proof of \cref{th:sm_jac_bound}}
\label{sec:proof_sm_jac_bound}

Utilizing the interlacing theorem, we can prove \cref{th:sm_jac_bound}.
\begin{proof}
Firstly, we rearrange coordinates of $x$ in such a way that 
\begin{equation}\label{eq:x_order}
x_1 \geqslant \dots \geqslant x_N.
\end{equation}
This rearrangement preserves the singular values, as it corresponds to the symmetric permutation of rows and columns of $A = \diag(x) - xx^\top$ which preserves matrix symmetry and singular values. In this case, we deal with positive semi-definite matrix, thus we can apply theorems for eigenvalues.

It can be seen that $A$ and $\diag(x)$ are both symmetric and $A$ is a rank-1 perturbation of the symmetric matrix $\diag(x)$.
Applying the eigenvalue interlacing theorem to estimate eigenvalues $\alpha_1 \geqslant \dots \geqslant \alpha_N$ of $A$ using the eigenvalues of $\diag(x)$, we obtain
\begin{equation}
    x_1 \geqslant \alpha_1 \geqslant \dots \geqslant x_N \geqslant \alpha_N.
\end{equation}
Now, without loss of generality, we can assume that $x_i \neq 0 ~ \forall \; i$, otherwise if $x_r = 0$ for some $r$, we can consider a leading minor of $A$ of size $r \times r$ and prove the result for it. Indeed, for $\{\alpha_i\}_{i=r}^{N}$ the statement of the theorem is fulfilled, as $\alpha_r = \dots = \alpha_N = 0$ by the interlacing theorem:
\begin{equation}
    \text{for $k \geqslant r$:} \quad 0 = \alpha_k \leqslant x_{r}(1 - x_{r} + x_{r+1}) = x_{k}(1 - x_{k} + x_{k+1}) = 0.
\end{equation}
Let us additionally assume that the scalar product of every eigenvector $y^{(k)}$ with $x$ is non-negative:
\begin{equation}\label{eq:all_positive}
\forall \, k: \quad \langle x, y^{(k)}\rangle \geqslant 0.
\end{equation}
We can always achieve this by flipping eigenvectors sign, if necessary.
Consider eigenvalue equation for the $k$-th largest eigenvalue:
\begin{equation}
    (\diag(x) - x x^{\top})y^{(k)} = \alpha_k y^{(k)},
\end{equation}
which is equivalent to the following
\begin{equation}
    x \odot y^{(k)} - \langle {x}, {y^{(k)}}\rangle_{{}}x = \alpha_k y^{(k)} 
\end{equation}
yielding
\begin{equation}\label{eq:eigenvalue_equation}
    (x - \alpha_k \mathbf{1})\odot y^{(k)} = \langle x, y^{(k)} \rangle x,
\end{equation}
where $\odot$ is element-wise product and $\mathbf{1} \in \mathbb{R}^N$ is the vector of all ones.
Multiplying both parts of \Cref{eq:eigenvalue_equation} by $\mathbf{1}^{\top}$ from left and using $\langle a \odot b, \mathbf{1} \rangle = \langle a, b\rangle$ for an arbitrary pair of vectors $a$ and $b$, we have:
\begin{equation}\label{eq:zero_sum}
\begin{split}
    \langle x - \alpha_k \mathbf{1}, y^{(k)}\rangle &= \mathbf{1}^{\top}\left((x - \alpha_k\mathbf{1})\odot y^{(k)}\right) = \mathbf{1}^{\top}x\langle x, y^{(k)}\rangle = \langle x, y^{(k)}\rangle \quad \Rightarrow\\
    \Rightarrow \quad 0 &= \alpha_k \langle {\mathbf{1}}, {y^{(k)}}\rangle_{{}} \quad \Rightarrow \quad \sum_{j=1}^{N} y^{(k)}_j = 0 \quad  \forall \, \alpha_k \neq 0.
\end{split}
\end{equation}
From \Cref{eq:eigenvalue_equation} it follows that elements of the vector $(x - \alpha_k \mathbf{1}) \odot y^{(k)}$ have the same sign. Indeed, using \Cref{eq:all_positive} we conclude that the RHS of \Cref{eq:eigenvalue_equation} consists of positive coordinates.
This assumption implies that the vector  $(x - \alpha_k \mathbf{1}) \odot y^{(k)}$ consists of non-negative coordinates from which it follows that we can find the signs of $y^{(k)}$.
Indeed, utilizing the interlacing theorem, the first $k$ coordinates of $(x - \alpha_k)$ are non-negative, while the rest are non-positive. Hence, the same condition for the signs is satisfied for $y^{(k)}$: the first $k$ coordinates are non-negative, while the rest are non-positive. Note that the same effect regarding eigenvectors' signs, but in general case was initially observed in \citep{fasino}.
Using this sign property for the coordinates of $y^{(k)}$, we can derive $\alpha_k$ from \Cref{eq:eigenvalue_equation}: for $k$-th coordinate of the \Cref{eq:eigenvalue_equation}, utilizing \Cref{eq:zero_sum} and the fact that $\mathbf{1}^{\top}y^{(k)} = 0$, we have
\begin{equation}
\begin{split}
    \alpha_k &= \frac{x_k\left(y_k^{(k)} - \langle {x}, {y^{(k)}}\rangle_{{}}\right)}{y_{k}^{(k)}} = x_k \left(1 - \frac{\langle {x}, {y^{(k)}}\rangle_{}}{y_k^{(k)}}\right) = \\ &= x_k\left(1 - \frac{\sum_{j=1}^{k} x_j y_{j}^{(k)}}{y_{k}^{(k)}} + \frac{\sum_{j=k+1}^{N}x_j(-y_j^{(k)})}{y_{k}^{(k)}}\right)\leqslant \\ & \leqslant x_k\left(1 - \frac{x_k\sum_{j=1}^{k} y_j^{(k)} - x_{k+1}\sum_{j=1}^{k} y_j^{(k)}}{y_k^{(k)}}\right) \leqslant x_k(1 - x_k + x_{k+1}) = g_k(x),
\end{split}
\end{equation}
which completes the proof.
\end{proof}

\section{Proof of \cref{th:attn_jac_bound}}
\label{sec:proof_attn_jac_bound}
\begin{proof}
In \citep{kim2021lipschitz}, the authors gave the exact form of the Jacobian of a single attention head: 
we can represent Jacobian matrix $\mathcal{J}_{\texttt{Attn}_h}(X)\in \mathbb{R}^{Nd \times ND}$ as a block matrix with $d \times D$ blocks, whose block $\mathcal{J}_{\texttt{Attn}_h}(X)_{ij} \in \mathbb{R}^{d \times D}$ admits a closed form expression:
\begin{equation}
\begin{split}
	\mathcal{J}_{\texttt{Attn}_h}(X)_{ij} = (W_h^{V})^{\top}X^{\top}\smjacmat{P^h_{i, :}}(E_{ji}XA_h^{\top} & + XA_h\delta_{ij}) + (W_h^{V})^\top \cdot P^h_{ij},
\end{split}
\end{equation}
where $P^h_{ij}$ is the $(i, j)$ element of the attention map, $A_h = \frac{W_h^Q(W_h^K)^\top}{\sqrt{d}}$, $E_{ji}$ is a binary matrix with zeros everywhere except for the $(j, i)$-th entry, 
$\delta_{ij}$ is the Kronecker symbol, and $\smjacmat{P^h_{i, :}}$ is the following matrix:
\begin{equation}
	\smjacmat{P^h_{i, :}} = \diag(P^h_{i,:}) - P_{i, :}^{h \top} P^h_{i, :}
\end{equation}
that corresponds to the softmax Jacobian at the point $z'$ such that $\texttt{sm}(z') = P^h_{i, :}$.
Using the proposed notation, we can express $\mathcal{J}_{\texttt{Attn}_h}(X)$ in the matrix form:
\begin{equation}
    \mathcal{J}_{\texttt{Attn}_h}(X) = \left(I_N \otimes (W^{V}_h)^{\top}\right)\left({\texttt{Block}(L_{ij})} + {\texttt{Block}(R_{ij})}\right) + P^h \otimes (W_h^{V})^{\top},
\end{equation}
where $\texttt{Block}(L_{ij})$ and $\texttt{Block}(R_{ij})$ represent block matrices with $d\times D$ blocks, whose $(i, j)$ blocks are equal to 
\begin{equation}\label{eq:block_defs}
L_{ij} = X^{\top}\smjacmat{P^h_{i, :}}E_{ji}XA_h^{\top}, \quad R_{ij} = X^{\top}\smjacmat{P^h_{i, :}}XA_h\delta_{ij} 
\end{equation}
respectively.
Now our goal is to bound $\|\mathcal{J}_{\texttt{Attn}_h}(X)\|_2$ utilizing the structure above.
Using triangle inequality and submultiplicativity of the spectral norm, we can write the following upper bound on this norm:
\begin{equation}
\begin{split}
    & \|{\mathcal{J}_{\texttt{Attn}_h}(X)}\|_{{2}} \leqslant \\ & \leqslant \left\|\left(I_N \otimes (W^V_h)^{\top}\right)\left({\texttt{Block}(L_{ij})}+ {\texttt{Block}(R_{ij})}\right)\right\|_{{2}} + \left\|{P^h \otimes (W_h^V)^{\top}}\right\|_{{2}} \leqslant \\ & \leqslant \|W_h^V\|_2 \cdot \left(\|\texttt{Block}(L_{ij})\|_2 + \|\texttt{Block}(R_{ij})\|_2\right) + \left\|{P^h \otimes (W_h^V)^{\top}}\right\|_{{2}}.
\end{split}
\end{equation}
We can bound $\texttt{Block}(L_{ij})$ and $\texttt{Block}(R_{ij})$ by submultiplicativity of the spectral norm: for $\texttt{Block}(L_{ij})$ we can represent block matrix as a product of simpler structured matrices:
\begin{equation}\label{eq:block_l}
\begin{split}
    {\texttt{Block}(L_{ij})} = \left(
	I_N \otimes X^{\top}
	\right) & \cdot 
        \blkdiag\left(\smjacmat{P^h_{1, :}}, \dots,\smjacmat{P^h_{N, :}}\right) \cdot \\ \cdot
        & \left(
            \begin{bmatrix}
        			E_{11}XA_h^{\top} & \dots & E_{N1}XA_h^{\top}\\
        			\vdots & \ddots & \vdots\\
        			E_{1N}XA_h^{\top} & \dots & E_{NN}XA_h^{\top}
            \end{bmatrix}
        \right),
\end{split}
\end{equation}
where $\blkdiag(B_1, \dots, B_N)$ denotes block-diagonal matrix, whose diagonal blocks are equal to $B_1, \dots B_N$ respectively.
Using this decomposition, we deduce the following bound: using $\|A \otimes B\|_2 = \|A\|_2\|B\|_2$ we obtain
\begin{equation}\label{eq:ineq_block_l}
\begin{split}
	& \|{\texttt{Block}(L_{ij})}\|_{{2}} 
    \leqslant \\ \leqslant & \left\|
	X^{\top}
	\right\|_{{2}}
	\left\|
        \blkdiag\left(\smjacmat{P^h_{1, :}}, \dots,\smjacmat{P^h_{N, :}}\right)
	\right\|_{{2}}
        \left\|
            \begin{bmatrix}
        			E_{11}XA_h^{\top} & \dots & E_{N1}XA_h^{\top}\\
        			\vdots & \ddots & \vdots\\
        			E_{1N}XA_h^{\top} & \dots & E_{NN}XA_h^{\top}
            \end{bmatrix}
        \right\|_{{2}}.
\end{split}
\end{equation}

This inequality can be simplified because the latter matrix in the equation above has a specific structure. By unitary invariance of the spectral norm, we can permute the rows of the latter matrix in \Cref{eq:ineq_block_l} preserving the spectral norm, thus obtaining a block-diagonal matrix out of this block matrix.
Indeed, by applying permutation, we can group all the first rows of $E_{11}XA_h^{\top}, \dots, E_{1N}XA_{h}^{\top}$ to the first diagonal block of size $d \times D$, all the second rows of $E_{11}XA_h^{\top}, \dots, E_{1N}XA_{h}^{\top}$ to the second diagonal block and so on. Thus,
\begin{equation}\label{eq:block_sparsity_norm}
\begin{split}
    \left\|{\begin{bmatrix}
			E_{11}XA_h^{\top} & \dots & E_{N1}XA_h^{\top}\\
			\vdots & \ddots & \vdots\\
			E_{1N}XA_h^{\top} & \dots & E_{NN}XA_h^{\top}
	\end{bmatrix}
	}\right\|_{2} & = 
     \left\|\blkdiag\left(XA^{\top}_h, \dots, XA_h^{\top}\right)\right\|_{2} = \\ & = \|I_N \otimes XA_h^{\top}\|_2 = \|XA^{\top}_h\|_2 \leqslant \|X\|_2\|A_h^{\top}\|_2.
\end{split}
\end{equation}
Substituting this equation to \Cref{eq:ineq_block_l} we obtain the final bound for the block matrix $\texttt{Block}(L_{ij})$:
\begin{equation}
\begin{split}
	\left\|{\texttt{Block}(L_{ij})}\right\|_{{2}} \leqslant \left\|{X}\right\|_{{2}}^2 \left\|{A_h}\right\|_{{2}} \max_i \left\|\smjacmat{P_{i, :}^h}\right\|_{{2}}.
\end{split}
\end{equation}
In the case of $\texttt{Block}(R_{ij})$, we have a similar to \Cref{eq:block_l} block matrix  which can be decomposed as follows:
\begin{equation}\label{eq:block_r}
    {\texttt{Block}(R_{ij})} = 
    \left(
    I_N \otimes X^{\top}
    \right)
    \blkdiag\left(\smjacmat{P^h_{1, :}}, \dots, \smjacmat{P^h_{N, :}}\right)
    \left(
    I_N \otimes X
    \right)
    \left(
    I_N \otimes A_h
    \right),
\end{equation}
which can be expressed as a product of block-diagonal factors as well to obtain the following bound:
\begin{equation}\label{eq:ineq_block_r}
\begin{split}
	& \left\|{\texttt{Block}(R_{ij})}\right\|_{{2}} \leqslant \\ \leqslant 
	& \left\|
	I_N \otimes X^{\top}
	\right\|_{{2}}
	\left\|\blkdiag\left(\smjacmat{P^h_{1, :}}, \dots, \smjacmat{P^h_{N, :}}\right)\right\|_{{2}}
	\left\|
	I_N \otimes X
	\right\|_{{2}}
	\left\|
	I_N \otimes A_h
	\right\|_{{2}}
\leqslant \\ \leqslant & \left\|X\right\|_2^2\|A_h\|_2\max_{i}\|\smjacmat{P_{i, :}^h}\|_2.    
\end{split}
\end{equation}

Now, combining bounds from \Cref{eq:ineq_block_l} and \Cref{eq:ineq_block_r}, we deduce the resulting upper bound for the spectral norm of the single-head self-attention Jacobian:
\begin{equation}\label{eq:main_inequality}
\begin{split}
    \left\|{\mathcal{J}_{\texttt{Attn}_h}(X)}\right\|_{{2}} \leqslant \left\|{W_h^V}\right\|_{{2}}\left(\left\|{P}\right\|_{{2}} + 2 \left\|{X}\right\|_{{2}}^2 \left\|{A_h}\right\|_{{2}}\underset{i}{\max}\left\|\smjacmat{P_{i, :}^h}\right\|_{{2}}\right).
\end{split}
\end{equation}
which completes the proof.
\end{proof}

Now, to clarify the bound
\begin{equation}
\begin{split}
    & \left\|
    \begin{bmatrix}
            \mathcal{J}_{\texttt{Attn}_1}(X) & \dots & \mathcal{J}_{\texttt{Attn}_H}(X)
    \end{bmatrix}
    \right\|_2 
    \leqslant \sum_{i=1}^{H}\left\|\mathcal{J}_{\texttt{Attn}_i}(X)\right\|_2.
\end{split}
\end{equation}
we can bound the spectral norm of multi-head self-attention Jacobian by the sum of the norms of its heads:
\begin{equation}
\begin{split}
    & \left\|
    \begin{bmatrix}
            \mathcal{J}_{\texttt{Attn}_1}(X) & \dots & \mathcal{J}_{\texttt{Attn}_H}(X)
    \end{bmatrix}
    \right\|_2 = \lambda_{\max}
        \left(
            \sum_{h=1}^{H} \mathcal{J}_{\texttt{Attn}_h}(X)\mathcal{J}_{\texttt{Attn}_h}(X)^\top
        \right)^{\frac 12} \leqslant \\ \leqslant & \sqrt{\sum_{h=1}^{H}\lambda_{\max}\left(\mathcal{J}_{\texttt{Attn}_h}(X)\mathcal{J}_{\texttt{Attn}_h}^{\top}(X)\right)} = \sqrt{\sum_{h=1}^{H}\left\|\mathcal{J}_{\texttt{Attn}_h}(X)\right\|^2_2}
    \leqslant \sum_{h=1}^{H}\left\|\mathcal{J}_{\texttt{Attn}_h}(X)\right\|_2.
\end{split}
\end{equation}

\section{Proof of Proposition~\ref{prop:attn_jac_bound_without_sm_jac}}\label{sec:proofattn_jac_bound_without_sm_jac}
\begin{proof}
To prove this proposition we use the same notation and idea as in the proof of \cref{th:attn_jac_bound} (see Appendix~\ref{sec:proof_attn_jac_bound} for more details), expressing the self-attention Jacobian matrix as a sum of two block matrices:

\begin{equation}
\begin{split}
    & \left\|{\mathcal{J}_{\texttt{Attn}_h}(X)}\right\|_{{2}} \leqslant \\ & \leqslant \left\|I_N \otimes (W^V_h)^{\top}\right\|_{{2}}\left(\left\|{\texttt{Block}(L_{ij})}\right\|_{{2}} + \left\|{\texttt{Block}(R_{ij})}\right\|_{{2}}\right) + \left\|{P^h \otimes (W_h^V)^{\top}}\right\|_{{2}},
\end{split}
\end{equation}
where $\texttt{Block}(L_{ij})$ and $\texttt{Block}(R_{ij})$ are defined in \Cref{eq:block_defs}.

Now, similarly to the previous proof, we can express $\texttt{Block}(L_{ij})$ as a product of several structured matrices:

\begin{equation}
\begin{split}
    & \texttt{Block}(L_{ij}) = \\ = & \blkdiag\left(X^{\top}\smjacmat{P^h_{1, :}}, \dots, X^{\top}\smjacmat{P^h_{N, :}}\right) 
    \left(
        \begin{bmatrix}
                E_{11}XA_h^{\top} & \dots & E_{N1}XA_h^{\top}\\
                \vdots & \ddots & \vdots\\
                E_{1N}XA_h^{\top} & \dots & E_{NN}XA_h^{\top}
        \end{bmatrix}
    \right).
\end{split}
\end{equation}
From \Cref{eq:block_sparsity_norm}, we already know the exact spectral norm of the latter matrix in the equation above. Thus, it remains only to estimate the spectral norm of the block-diagonal matrix $\blkdiag(X^{\top}\smjacmat{P^h_{1, :}}, \dots, X^{\top}\smjacmat{P^h_{N, :}})$. Now, since the spectral norm of a block-diagonal matrix equals the maximum spectral norm of its diagonal blocks, one may decompose $\smjacmat{P_{i, :}^h}$ as follows:
\begin{equation}
\smjacmat{P^h_{i, :}} = \diag\left(\sqrt{P^h_{i, :}}\right)\left(I_N - \sqrt{P^{h, \top}_{i, :}}\sqrt{P^h_{i, :}}\right)\diag\left(\sqrt{P^h_{i, :}}\right)
\end{equation}
and, using this representation, bound $\|X^{\top}\smjacmat{P_{i, :}^h}\|_2$ as follows:
\begin{equation}
\begin{split}
    \|X^{\top}\smjacmat{P^h_{i, :}}\|_2 \leqslant & \left\|X^{\top}\diag\left(\sqrt{P^h_{i, :}}\right)\right\|_2 \left\|I_N - \sqrt{P^{h \top}_{i, :}}\sqrt{P^{h}_{i, :}}\right\|_2 \left\|\diag(\sqrt{P^h_{i, :}})\right\|_2
    \leqslant \\ \leqslant &  \left\|X^{\top}\diag\left(\sqrt{P^h_{i, :}}\right)\right\|_2 \cdot 1 \cdot 1.
\end{split}
\end{equation}
Indeed, $I_N - \sqrt{P^{h \top}_{i, :}}\sqrt{P^{h}_{i, :}}$ is an orthogonal projection matrix and $\diag\left(\sqrt{P^h_{i, :}}\right)$ is a diagonal matrix whose entries do not exceed $1$, which implies that its spectral norms does not exceed 1.
Now it remains to bound $\left\|X^{\top}\diag\left(\sqrt{P_{i,:}^h}\right)\right\|_2$ by the definition of the spectral norm, we can bound $X^{\top}\diag\left(\sqrt{P^h_{i, :}}\right)$ as follows:
\begin{equation}
\begin{split}
    \left\|X^{\top}\diag\left(\sqrt{P^h_{i, :}}\right)\right\|_2^2 &=  \sup_{\|y\|_2=1} \|\diag\left(\sqrt{P^h_{i, :}}\right)Xy\|_2^2
    \leqslant
    \sup_{\|y\|_2=1}\sum_{j=1}^{N} P^h_{i, j}\langle X_{j, :}, y\rangle^2 \leqslant \\ & \leqslant \sum_{j=1}^{N} P^h_{i, j} R^2 = R^2.
\end{split}
\end{equation}
Now we can derive similar to \Cref{eq:block_r} upper bound for $\texttt{Block}(L_{ij})$:

\begin{equation}
\begin{split}
    & \left\|\texttt{Block}(L_{ij})\right\|_2 \leqslant \\ \leqslant &
    \left\|
        \blkdiag\left(X^{\top}\smjacmat{P^h_{1, :}}, \dots, X^{\top}\smjacmat{P^h_{N, :}}\right)
    \right\|_2 \cdot
    \left\|{\begin{bmatrix}
			E_{11}XA_h^{\top} & \dots & E_{N1}XA_h^{\top}\\
			\vdots & \ddots & \vdots\\
			E_{1N}XA_h^{\top} & \dots & E_{NN}XA_h^{\top}
	\end{bmatrix}
    }\right\|_{2} \leqslant \\
    \leqslant & \max_{i} \|X^{\top}\smjacmat{P^h_{i, :}}\|_2  \|I _N\otimes XA^{\top}_h\|_2 \leqslant R \cdot \|X\|_2 \|A_h\|_2 \leqslant \sqrt{N}R^2 \|A_h\|_2,
\end{split}
\end{equation}
Similarly, for $\texttt{Block}(R_{ij})$, we obtain:
\begin{align}
\begin{split}
     & \left\|\texttt{Block}(R_{ij})\right\|_2 \leqslant \\ \leqslant & \left\|
        \blkdiag\left(X^{\top}\smjacmat{P^h_{1, :}}, \dots, X^{\top}\smjacmat{P^h_{N, :}}\right)
    \right\|_2 \|I_N \otimes X\|_2  \|I_N \otimes A_h\|_2 \leqslant \\ \leqslant & \max_{i}\|X^{\top}\smjacmat{P^h_{i, :}}\|_2 \|X\|_2\|A_h\|_2 \leqslant \sqrt{N}R^2 \|A_h\|_2.
\end{split}
\end{align}
Finally, combining all the modified bounds, we can write down the upper bound in terms of \citep{castin2023smooth}:
\begin{equation}\label{eq:castin_improve}
\begin{split}
    \|\mathcal{J}_{\texttt{Attn}_h}(X)\|_2 & \leqslant \|W_h^V\|_2 (\|P\|_2 + 2\sqrt{N}\|A_h\|_2R^2) \leqslant \|W_h^V\|_2(\sqrt{N} + 2\sqrt{N}R^2\|A_h\|_2).
\end{split}
\end{equation}
\end{proof}

Directly comparing the bound from \cref{th:how_smooth_is_attention} with the bound from \Cref{eq:castin_improve} leads to the following comparison:
\begin{equation}
\begin{split}
    3 \cdot (\|A_h\|_2^2R^4(4N+1)+N) &\geqslant N + 4NR^4\|A_h\|_2^2 + 4NR^2\|A_h\|_2 \\
\end{split}
\end{equation}
which is, after reduction of similar terms equivalent to
\begin{equation}
\begin{split}
    8N\|A_h\|_2^2R^4 + 2N + 3\|A_h\|_2^2R^4  &\geqslant 4NR^2\|A_h\|_2,
\end{split}
\end{equation}
from which, by AM-GM inequality, it follows that in practice the bound from Proposition~\ref{prop:attn_jac_bound_without_sm_jac} surpasses the bound from \citep{castin2023smooth}.

\section{Proof of Proposition~\ref{prop:uniform_on_k}}
\label{sec:proof_of_prop_1}
\begin{proof}
It is known that the $k$-th order statistic of a probability distribution is at most $1/k$.
Then, we can write the following chain of inequalities:
\begin{equation}
\begin{split}
& (P_{i, :}^{l, h})_{(1)}\left(1 - (P_{i, :}^{l, h})_{(1)}\right) \leqslant (P_{i, :}^{l, h})_{(1)}\left(1 - (P_{i, :}^{l, h})_{(1)} + (P_{i, :}^{l, h})_{(2)}\right) = g_1\left(P_{i, :}^{l, h}\right) \leqslant \\ \leqslant & \gamma \cdot g_k\left(P_{i, :}^{l, h}\right) = \gamma \cdot (P_{i, :}^{l, h})_{(k)}\left(1 -(P_{i, :}^{l, h})_{(k)} +(P_{i, :}^{l, h})_{(k+1)}\right) \leqslant \frac{1}{k} \cdot \gamma \cdot 1,
\end{split}
\end{equation}
which, via solving quadratic inequality, implies that
\begin{equation}
    (P_{i, :}^{l, h})_{(1)} \leqslant \frac{1 - \sqrt{1 - {4\gamma}/{k}}}{2} \; \lor \; (P_{i, :}^{l, h})_{(1)} \geqslant \frac{1 + \sqrt{1 - {4\gamma}/{k}}}{2}.
\end{equation}
\end{proof}
Now let us observe the case of $g_1 = g_k$, \ie the case when $\gamma = 1$.
From \cref{th:sm_jac_bound}, we can immediately conclude that $x_{(2)} = \dots = x_{(k)}$.
On the other hand, utilizing $g_1 = g_k$, we obtain:
\begin{equation}
\begin{split}
x_{(1)} (1 - x_{(1)} + x_{(2)}) & = 
x_{(k)} (1 - x_{(k)} + x_{(k+1)})  = x_{(2)}(1 - x_{(2)} + x_{(k+1)}).
\end{split}
\end{equation}
By using $x_{(1)}x_{(2)} \geqslant x_{(2)}x_{(k+1)}$, we arrive at:
\begin{equation}
\begin{split}
& x_1 - x_1^2 \leqslant x_2 - x_2^2,\\
& (x_{(1)} - x_{(2)})(1 - x_{(1)} - x_{(2)}) \leqslant 0.
\end{split}
\end{equation}
Since both terms in the latter product are nonnegative, we finally obtain either $x_{(1)} = 1$ or $x_{(1)} = x_{(2)} = \dots = x_{(k)}$. 
This would imply that the minimization of $g_1/g_k$ enforces attention map to distribute attention uniformly \emph{on at least $k$ tokens}. The categorical distribution $x_{(1)} = 1$ is formally prohibited in this case due to division by zero. 
To counteract it, we add a small $\epsilon \approx 1e\!-\!6$ to the denominator. Thus, in case of close proximity to the categorical distribution, we can still obtain a small regularization term. 
Otherwise, our loss promotes a uniform distribution on at least $k$ tokens.

Thus, when we decrease $k$, we impose a less strict constraint on the distribution of attention scores among tokens, which can serve as a tool to alleviate the oversmoothing problem~\citep{shi2022revisitingoversmoothingbertperspective}  in comparison to the uniform on all tokens.

\section{Matrix form of the dot-product self-attention Jacobian}\label{sec:matrix_jacobian_derivations}
Apart from block matrix view on the self-attention Jacobian, we also provide below the derivation of a closed-form matrix formula for the self-attention Jacobian matrix.

As attention can be viewed as a composition of matrix functions, we derive its Jacobian in a sequential manner. Let $X \in \mathbb{R}^{N \times D}$ and $W_h^Q, W_h^K, W_h^V \in \mathbb{R}^{D \times d}$. 
Importantly, in this section and in Appendix~\ref{sec:proof_gradient_norms_from_attention_jac} we use a different notation to explicitly derive Jacobian matrices \wrt different parameters, for example, model weights:
\begin{equation}
\begin{split}
    Q_h(X, W_h^Q) \coloneqq X W_h^Q, \quad K_h(X, W_h^K) &\coloneqq X W_h^K, \quad V_h(X, W_h^V) \coloneqq X W_h^V, \\
    QK_h(X, W_h^Q, W_h^K) \coloneqq & \frac{Q_h(X)K_h(X)^\top}{\sqrt{d}} = X A_h X^\top, \\ P^h(X, W_h^Q, W_h^K) \coloneqq & \: \texttt{sm}(QK_h(X, W_h^Q, W_h^K)).
\end{split}
\end{equation}
Then, $\texttt{Attn}_h$ can be written using this notation as follows:
\begin{equation}
    \texttt{Attn}_h(X, W_h^Q, W_h^K, W_h^V) = P^h(X, W_h^Q, W_h^K)V_h(X, W_h^V).
\end{equation}
We omit the functional dependencies for intermediate terms in later derivations for notation simplicity. Finally, we also denote $d_{X}f(\dotsc, X, \dotsc)$ as the differential of $f$ \wrt $X$.

Using a ``c''-order identity for Kronecker product vectorization $\mathrm{vec}(AXB) = (A \otimes B^\top) \mathrm{vec}(X)$ and a differential expression for the Jacobian $\mathrm{vec}(d_X F(X)) = \mathcal{J}_F(X) \mathrm{vec}(d_X X)$ we may write:
\begin{equation}
    \mathrm{vec}(d_X Q_h) = (I_N \otimes (W_h^Q)^{\top}) \mathrm{vec}(dX) 
    \quad \Longrightarrow \quad \mathcal{J}_{Q_h}(X) = (I_N \otimes (W_h^Q)^{\top})
\end{equation}
and, similarly,

\begin{equation}
    \mathcal{J}_{K_h}(X) = (I_N \otimes (W_h^K)^{\top}), \quad \mathcal{J}_{V_h}(X) = (I_N \otimes (W_h^V)^{\top}).
\end{equation}

Recall that Kronecker product admits the following property \citep{golub2013matrix}: for any pair of $A \in \mathbb{R}^{N \times N}, B \in \mathbb{R}^{D \times d}$: 
\begin{equation}
\mathcal{P}_{N, D}(A \otimes B) = (B \otimes A) \mathcal{P}_{N, d},
\end{equation}
where $\mathcal{P}_{\star,\star}$ is a ``perfect shuffle'' permutation matrix \citep{golub2013matrix}. Then, using the following property about ``perfect shuffle'' permutation: for an arbitrary $M \in \mathbb{R}^{m \times n}: \mathcal{P}_{m, n}\mathrm{vec}(M) = \mathrm{vec}(M^{\top})$, we obtain
\begin{equation}
\begin{split}
     \mathrm{vec}(d_X QK_h) &= \frac{1}{\sqrt{d}}\left((I_N \otimes K_h) \mathrm{vec}(d_X Q_h) + (Q_h \otimes I_N) \mathrm{vec}(d_X K_h^\top)\right) = \\ & =
     \frac{1}{\sqrt{d}}\Big((I_N \otimes XW_h^K)(I_N \otimes (W_h^{Q})^\top) + \\ & + (XW_h^Q \otimes I_N) \mathcal{P}_{N, d} (I_N \otimes (W_h^{K})^{\top})\Big)\mathrm{vec}(d_X X) = \\ &=
     (I_N \otimes XA_h^\top + (XA_h \otimes I_N) \mathcal{P}_{N, D}) \mathrm{vec}(d_X X).
\end{split}
\end{equation}
Then, utilizing composition differentiating rule, we can derive softmax Jacobian:
\begin{equation}\label{eq:bdiag_softmax_jac}
\begin{split}
    \mathrm{vec}(d_X P^h) = \blkdiag\left(\smjacmat{P^h_{1, :}}, \dots,\smjacmat{P^h_{N, :}}\right) \mathrm{vec}(d_X QK_h)
\end{split}
\end{equation}
and substitute it into the whole self-attention Jacobian expression \wrt $X$:
\begin{equation}
\begin{split}
    &\mathrm{vec}(d_X \texttt{Attn}_h) = (I_N \otimes V_h^\top(X)) \mathrm{vec}(d_X P^h) + (P^h \otimes I_d)\mathrm{vec}(d_X V_h) = \\ = &
    (I_N \otimes (W_h^V)^\top X^\top) \mathrm{vec}(d_X P^h) + (P^h \otimes I_d)(I_N \otimes (W_h^{V})^\top)\mathrm{vec}(d_X X).
\end{split}
\end{equation}
Thus, we obtain the final matrix form of the dot-product self-attention Jacobian:
\begin{equation}\label{eq:softmax_jac_maxtrix_form}
\begin{split}
    \mathcal{J}_{\texttt{Attn}_h}(X) = &(I_N \otimes (W_h^V)^\top X^\top)\blkdiag\left(\smjacmat{P^h_{1, :}}, \dots,\smjacmat{P^h_{N, :}}\right)\cdot \\ \cdot &(I_N \otimes XA_h^\top + (XA_h \otimes I_N)\mathcal{P}_{N, D}) + (P^h \otimes (W_h^{V})^\top).
\end{split}
\end{equation}

\section{Proof of Proposition~\ref{prop:gradient_norms_from_attention_jac}}\label{sec:proof_gradient_norms_from_attention_jac}
\begin{proof}
    Using the technique, similar to the one outlined above in Appendix~\ref{sec:matrix_jacobian_derivations}, we obtain analytic forms for Jacobians w.r.t. the self-attention weight matrices.

    First, let's notice that  $\texttt{Attn}_h(X, W_h^Q, W_h^K, W_h^V)$ depends on $W_h^V$ only through the term $V_h$, thus we have 
\begin{equation}
    \mathrm{vec}(d_{W_h^V} \texttt{Attn}_h) = (P^hX \otimes I_d) \mathrm{vec}(d_{W_h^V} W_h^V) \quad \Rightarrow \quad \frac{\partial \texttt{Attn}_h}{\partial W_h^V} = (P^hX \otimes I_d).
\end{equation}
Conversely, $\texttt{Attn}_h(X, W_h^Q, W_h^K, W_h^V)$ depends on $W_h^Q, W_h^K$ only through $P^h$. Thus, for $W_h^Q$ we have:
\begin{equation}
\begin{split}
    \mathrm{vec}(d_{W_h^Q} Q_h) &= (X \otimes I_d) \mathrm{vec}(d_{W_h^Q} W_h^Q) \\
    \mathrm{vec}(d_{W_h^Q} QK_h) &= \frac{1}{\sqrt{d}}(I_N \otimes K_h) \mathrm{vec}(d_{W_h^Q} Q_h) = \\ &= \frac{1}{\sqrt{d}}(I_N \otimes XW_h^K)(X \otimes I_d)\mathrm{vec}(d_{W_h^Q} W_h^Q) = \\ &= \frac{1}{\sqrt{d}}(X \otimes XW_h^K) \mathrm{vec}(d_{W_h^Q} W_h^Q).
\end{split}
\end{equation}
Then, using \Cref{eq:bdiag_softmax_jac}, we obtain a similar softmax score-dependent Jacobian expression for the weight matrix:
\begin{equation}
\begin{split}
    \mathrm{vec}(d_{W_h^Q} P^h) = \cfrac{1}{\sqrt{d}}\blkdiag\left(\smjacmat{P^h_{1, :}}, \dots,\smjacmat{P^h_{N, :}}\right) (X \otimes XW_h^K) \mathrm{vec}(d_{W_h^Q} W_h^Q)
\end{split}
\end{equation}
and, finally:
\begin{equation}
\begin{split}
    \mathrm{vec}(d_{W_h^Q}\texttt{Attn}_h) &= \cfrac{1}{\sqrt{d}}(I_N \otimes (W_h^V)^\top X^\top)\blkdiag\left(\smjacmat{P^h_{1, :}}, \dots,\smjacmat{P^h_{N, :}}\right) \cdot \\ &\cdot (X \otimes XW_h^K) \mathrm{vec}(d_{W_h^Q} W_h^Q).
\end{split}
\end{equation}
A Jacobian expression w.r.t. $W_h^K$ is obtained in a similar manner:
\begin{equation}
    \begin{split}
        \mathrm{vec}(d_{W_h^K}\texttt{Attn}_h) &= \cfrac{1}{\sqrt{d}}(I_N \otimes (W_h^V)^\top X^\top)\blkdiag\left(\smjacmat{P^h_{1, :}}, \dots,\smjacmat{P^h_{N, :}}\right) \cdot \\ &\cdot (XW_h^Q \otimes X)  \mathcal{P}_{D, d}\mathrm{vec}(d_{W_h^K} W_h^K).
    \end{split}
\end{equation}

Now, using the fact that $\|A \otimes B\|_2 = \|A\|_2\|B\|_2$, $\mathcal{P}_{D,d}$ is a permutation matrix, so $\left\| \mathcal{P}_{D,d} \right\|_2 = 1$, and submultiplicativity of the spectral norm, we can derive upper bounds on the spectral norms of self-attention Jacobians \wrt $W_h^Q, W_h^K, W_h^V$:

\begin{equation}
\begin{split}
    \left\|\frac{\partial \texttt{Attn}_h}{\partial W_h^Q}\right\|_2 & \leqslant \frac{1}{\sqrt{d}} \left\|W_h^V\right\|_2 \max_{i=1,\dotsc,N} \left\|\smjacmat{P_{i,:}^h}\right\|_2 \|X\|_2^3 \left\|W_h^K\right\|_2,\\
    \left\|\frac{\partial \texttt{Attn}_h}{\partial W_h^K}\right\|_2 & \leqslant \frac{1}{\sqrt{d}} \left\|W_h^V\right\|_2 \max_{i=1,\dotsc,N} \left\|\smjacmat{P_{i,:}^h}\right\|_2 \|X\|_2^3 \left\|W_h^Q\right\|_2,\\
    \left\|\frac{\partial \texttt{Attn}_h}{\partial W_h^V}\right\|_2 & \leqslant \left\| P^h \right\|_2 \left\| X \right\|_2 \leqslant \sqrt{\left\| P^h \right\|_1 \left\| P^h \right\|_\infty} \left\| X \right\|_2 \leqslant \sqrt{\max_j \textstyle{\sum_{i=1}^N} P^h_{ij}} \ \left\| X \right\|_2.
\end{split}
\end{equation}
    
\end{proof}

\section{Additional experimental details}
\label{sec:add_experiments}

All the training experiments are conducted using NVIDIA V100-SXM2-32GB GPU.
Other runs (\eg empirical observations) were conducted on NVIDIA A100-SXM2-80GB GPU. We ran all the experiments within $\sim8000$ GPU hours overall.

For training our models, we employ Stochastic Gradient Descent (SGD) with a weight decay of $10^{-4}$ and momentum equal to $0.9$. We train our models for $40$ epochs, with batch size equal to $512$ for CIFAR-10 and CIFAR-100 and $128$ for Imagenette. In order to have fair a comparison with \texttt{Specformer} \citep{hu2024specformer}, we remove all the bias terms in query, key and value parameters.
To improve convergence, we utilize a cosine annealing learning rate schedule \citep{loshchilov2017sgdrstochasticgradientdescent}. 
By default, we incorporate standard augmentation techniques, including horizontal flips, CutMix \citep{yun2019cutmix} and MixUp \citep{zhang2017mixup}.
In \Cref{tab:ablation_results_cifar100} and \Cref{tab:ablation_results_cifar10} we report an extended hyperparameter sweep for CIFAR-10 and CIFAR-100 datasets.

To show the scalability of our theoretical upper bound on the spectral norm of the self-attention Jacobian, we extend its comparison from 
\Cref{fig:vit_l_consts} 
to a larger-scale model and provide a graph for the visual part of Qwen3-VL-4B \cite{yang2025qwen3technicalreport}. The visual part of the model has 24 transformer blocks with 16 attention heads. We sample 100 image samples from OmniGAIA benchmark \cite{li2026omnigaia}, extract the attention scores from the visual part, and plot the heatmap with average logarithmic value of the self-attention Jacobian spectral norm upper bounds. Each image sample from the benchmark is transformed into a sequence of 1500 visual tokens. \Cref{fig:qwen_lip_consts} shows that even for larger sequence lengths, our theoretical upper bound is able to capture more information from the self-attention mechanism than previous theoretical approximations. We do not report the true mean value of the self-attention Jacobian spectral norm in this case because of its computational burden. But still, our bound consistently outperforms both theoretical bounds from \cref{th:how_smooth_is_attention} and \cref{th:specformer}.

\begin{figure}
    \centering
    \resizebox{0.8\textwidth}{!}{
    \begin{tikzpicture}
    \node (picture) at (0, 0) {\includegraphics[width=1\linewidth]{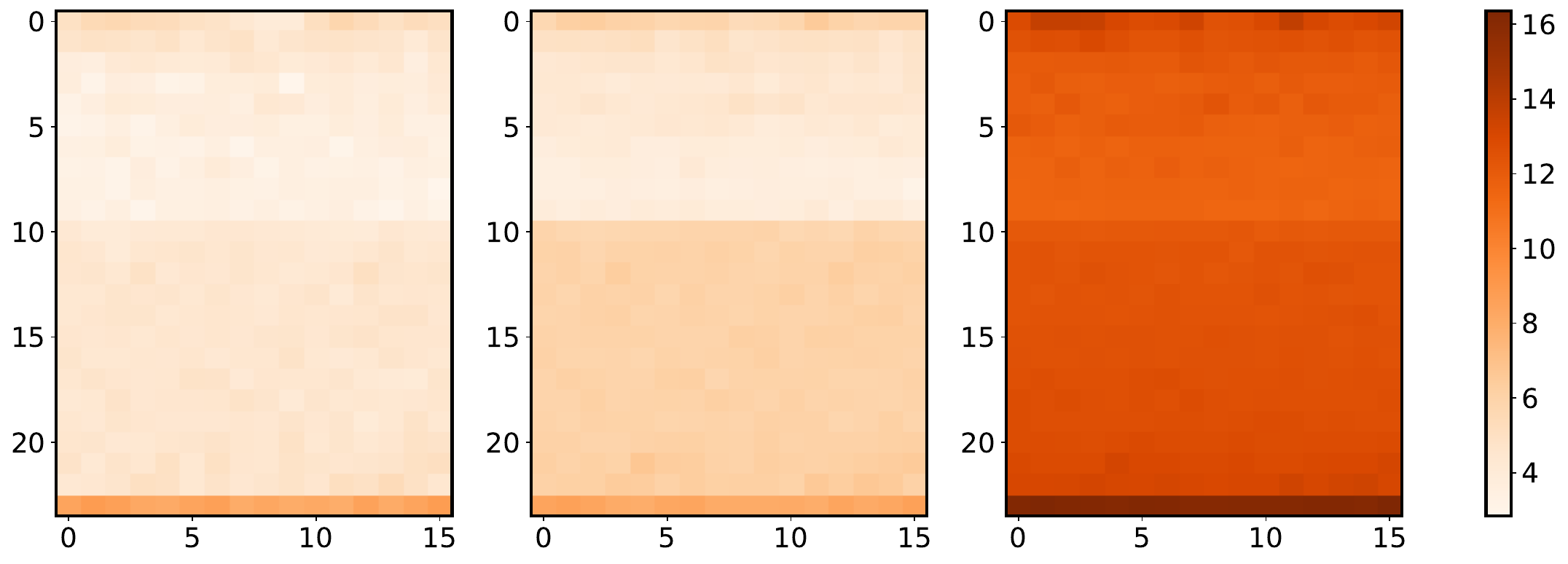}};
    \node[inner sep=0pt] (l1) at (-4.15, 2.5) {\tiny \Cref{th:attn_jac_bound}};
    \node[inner sep=0pt] (l2) at (-4.15, 2.3) {\tiny Mean: 2.1e4}; 
    \node[inner sep=0pt] (l1) at (-0.45, 2.5) {\tiny \Cref{th:how_smooth_is_attention}};
    \node[inner sep=0pt] (l2) at (-0.45, 2.3) {\tiny Mean: 2.3e5}; 
    \node[inner sep=0pt] (l1) at (3.25, 2.5) {\tiny \Cref{th:specformer}}; 
    \node[inner sep=0pt] (l2) at (3.25, 2.3) {\tiny Mean: 2.7e12}; 
    \node[inner sep=0pt] (h1) at (-4.15, -2.2) {\tiny Head};
    \node[inner sep=0pt] (h1) at (-0.45, -2.2) {\tiny Head};
    \node[inner sep=0pt] (h1) at (3.25, -2.2) {\tiny Head};
    \node[inner sep=0pt, rotate=90] (l3) at (-6.25, 0) {\tiny Layer};
    \end{tikzpicture}
    }
    \caption{Self-attention Jacobian spectral norms' comparison for each head and layer of the Qwen3-VL-4B \citep{yang2025qwen3technicalreport} vision part averaged across 100 OmniGAIA benchmark \citep{li2026omnigaia} samples. The title of each heatmap represents the mean value of the bound across all heads and layers.}
    \label{fig:qwen_lip_consts}
\end{figure}

\begin{table*}[t]
  \centering
    \caption{Hyperparameter sweep for ViT-B on CIFAR-100 dataset. A number after the name of the attack denotes attack budget. We evaluate PGD and AutoAttack with 4 and 10 steps respectively.}
  \label{tab:ablation_results_cifar100}
  \begin{tabular}{lcccccc}
    \toprule
    \multicolumn{1}{c}{Method} & \multicolumn{6}{c}{CIFAR-100} \\
    \cmidrule(lr){2-7} 
    & Standard & FGSM2 & FGSM4 & PGD2 & PGD4 & AA2 \\
    \midrule
    $\texttt{Baseline}$ & $90.90$ & $46.39$ & $39.54$ & $26.85$ & $12.91$ & $2.12$ \\
    $\texttt{Specformer}_{(10^{-2},10^{-2}, 10^{-2})}$ & $90.02$ & $44.12$ & $37.04$ & $27.55$ & $12.57$ & $3.01$ \\
    $\texttt{Specformer}_{(10^{-2}, 0, 0)}$ & $90.97$ & $45.87$ & $39.08$ & $28.38$ & $12.74$ & $2.61$ \\
    $\texttt{Specformer}_{(0, 10^{-2}, 0)}$ & $90.64$ & $46.04$ & $39.46$ & $29.29$ & $13.34$ & $2.57$ \\
    $\texttt{Specformer}_{(0, 0, 10^{-2})}$ & $90.91$ & $47.00$ & $40.68$ & $29.39$ & $14.28$ & $2.39$ \\
    $\texttt{Specformer}_{(10^{-3},10^{-3}, 10^{-3})}$ & $90.92$ & $46.52$ & $39.72$ & $28.13$ & $12.97$ & $2.33$ \\
    $\texttt{Specformer}_{(10^{-3}, 0, 0)}$ & $91.04$ & $47.60$ & $41.12$ & $28.22$ & $13.36$ & $2.16$ \\
    $\texttt{Specformer}_{(0, 10^{-3}, 0)}$ & $90.67$ & $46.13$ & $39.52$ & $27.00$ & $12.31$ & $2.13$ \\
    $\texttt{Specformer}_{(0, 0, 10^{-3})}$ & $91.04$ & $47.48$ & $40.79$ & $28.27$ & $13.14$ & $2.10$ \\
    $\texttt{Specformer}_{(10^{-4},10^{-4}, 10^{-4})}$ & $90.91$ & $47.18$ & $41.07$ & $28.34$ & $13.38$ & $2.22$ \\
    $\texttt{Specformer}_{(10^{-4}, 0, 0)}$ & $90.95$ & $46.73$ & $39.73$ & $27.64$ & $12.91$ & $1.90$ \\
    $\texttt{Specformer}_{(0, 10^{-4}, 0)}$ & $\mathbf{91.25}$ & $46.91$ & $40.48$ & $28.57$ & $13.37$ & $1.98$ \\
    $\texttt{Specformer}_{(0, 0, 10^{-4})}$ & $90.96$ & $46.92$ & $40.20$ & $27.97$ & $13.03$ & $2.22$ \\
    $\texttt{JaSMin}_{k=0, \lambda=10^{-2}}$ & $88.85$ & $48.09$ & $41.47$ & $\textbf{32.71}$ & $\textbf{19.12}$ & $\textbf{3.83}$ \\
    $\texttt{JaSMin}_{k=0, \lambda=10^{-3}}$ & $90.92$ & $\underline{48.41}$ & $41.98$ & $30.00$ & $14.91$ & $2.34$ \\
    $\texttt{JaSMin}_{k=0, \lambda=10^{-4}}$ & $90.90$ & $47.38$ & $41.02$ & $28.71$ & $13.64$ & $2.01$ \\
    $\texttt{JaSMin}_{k=10, \lambda=10^{-2}}$ & $87.81$ & $46.17$ & $38.78$ & $\underline{30.82}$ & $\underline{18.25}$ & $\underline{3.49}$ \\
    $\texttt{JaSMin}_{k=10, \lambda=10^{-3}}$ & $90.94$ & $\textbf{48.72}$ & $\textbf{42.28}$ & $29.99$ & $15.39$ & $2.48$ \\
    $\texttt{JaSMin}_{k=10, \lambda=10^{-4}}$ & $90.89$ & $47.14$ & $40.59$ & $28.50$ & $13.71$ & $2.17$ \\
    $\texttt{JaSMin}_{k=30, \lambda=10^{-2}}$ & $86.52$ & $42.76$ & $36.03$ & $26.90$ & $14.41$ & $2.54$ \\
    $\texttt{JaSMin}_{k=30, \lambda=10^{-3}}$ & $90.58$ & $48.27$ & $\underline{42.17}$ & $30.12$ & $15.28$ & $2.44$ \\
    $\texttt{JaSMin}_{k=30, \lambda=10^{-4}}$ & $\underline{91.10}$ & $46.90$ & $40.53$ & $27.45$ & $12.52$ & $2.00$ \\
    \bottomrule
  \end{tabular}
\end{table*}

\begin{table*}[t]
  \centering
    \caption{Hyperparameter sweep for ViT-B on CIFAR-10 dataset. The number after the name of the attack denotes attack budget. We evaluate PGD and AutoAttack with 4 and 10 steps respectively.}
  \label{tab:ablation_results_cifar10}
  \begin{tabular}{lcccccc}
    \toprule
    \multicolumn{1}{c}{Method} & \multicolumn{6}{c}{CIFAR-10} \\
    \cmidrule(lr){2-7} 
    & Standard & FGSM2 & FGSM4 & PGD2 & PGD4 & AA2 \\
    \midrule
    $\texttt{Baseline}$ & $\textbf{98.56}$ & $74.72$ & $67.17$ & $45.45$ & $19.71$ & $2.13$ \\
    $\texttt{Specformer}_{(10^{-2}, 10^{-2}, 10^{-2})}$ & $97.72$ & $71.38$ & $64.64$ & $\underline{49.92}$ & $\underline{24.18}$ & $\textbf{4.38}$ \\
    $\texttt{Specformer}_{(10^{-2}, 0, 0)}$ & $98.24$ & $73.00$ & $65.02$ & $44.80$ & $17.75$ & $2.87$ \\
    $\texttt{Specformer}_{(0, 10^{-2}, 0)}$ & $98.27$ & $73.88$ & $67.39$ & $47.88$ & $21.34$ & $\underline{3.13}$ \\
    $\texttt{Specformer}_{(0, 0, 10^{-2})}$ & $98.35$ & $74.06$ & $66.79$ & $47.74$ & $21.40$ & $2.46$ \\
    $\texttt{Specformer}_{(10^{-3}, 10^{-3}, 10^{-3})}$ & $98.36$ & $74.76$ & $68.61$ & $47.48$ & $21.37$ & $2.59$ \\
    $\texttt{Specformer}_{(10^{-3}, 0, 0)}$ & $98.36$ & $73.77$ & $66.65$ & $45.12$ & $18.93$ & $2.23$ \\
    $\texttt{Specformer}_{(0, 10^{-3}, 0)}$ & $98.41$ & $74.06$ & $67.05$ & $44.57$ & $19.30$ & $2.09$ \\
    $\texttt{Specformer}_{(0, 0, 10^{-3})}$ & $\underline{98.55}$ & $74.01$ & $67.63$ & $44.81$ & $19.16$ & $1.85$ \\
    $\texttt{Specformer}_{(10^{-4}, 10^{-4}, 10^{-4})}$ & $98.46$ & $\underline{75.30}$ & $\textbf{68.71}$ & $46.33$ & $19.82$ & $2.10$ \\
    $\texttt{Specformer}_{(10^{-4}, 0, 0)}$ & $98.45$ & $74.50$ & $67.46$ & $47.20$ & $20.40$ & $2.30$ \\
    $\texttt{Specformer}_{(0, 10^{-4}, 0)}$ & $98.46$ & $74.62$ & $67.61$ & $46.38$ & $20.39$ & $2.08$ \\
    $\texttt{Specformer}_{(0, 0, 10^{-4})}$ & $98.51$ & $74.42$ & $67.46$ & $45.85$ & $20.29$ & $2.04$ \\
    $\texttt{JaSMin}_{k=0, \lambda=10^{-2}}$ & $96.71$ & $65.59$ & $57.29$ & $44.98$ & $24.17$ & $\textbf{4.38}$ \\
    $\texttt{JaSMin}_{k=0, \lambda=10^{-3}}$ & $98.52$ & $74.80$ & $67.74$ & $46.53$ & $19.50$ & $2.23$ \\
    $\texttt{JaSMin}_{k=0, \lambda=10^{-4}}$ & $98.37$ & $74.59$ & $67.18$ & $45.17$ & $19.47$ & $2.10$ \\
    $\texttt{JaSMin}_{k=10, \lambda=10^{-2}}$ & $96.49$ & $66.67$ & $58.75$ & $45.66$ & $23.74$ & $\underline{3.13}$ \\
    $\texttt{JaSMin}_{k=10, \lambda=10^{-3}}$ & $98.21$ & $\textbf{75.76}$ & $\underline{69.10}$ & $\mathbf{51.64}$ & $\textbf{26.35}$ & $2.79$ \\
    $\texttt{JaSMin}_{k=10, \lambda=10^{-4}}$ & $98.40$ & $73.52$ & $65.92$ & $42.67$ & $17.33$ & $1.69$ \\
    $\texttt{JaSMin}_{k=30, \lambda=10^{-2}}$ & $95.70$ & $60.86$ & $50.87$ & $36.70$ & $14.69$ & $2.39$ \\
    $\texttt{JaSMin}_{k=30, \lambda=10^{-3}}$ & $98.09$ & $74.48$ & $67.59$ & $48.17$ & $22.91$ & $2.90$ \\
    $\texttt{JaSMin}_{k=30, \lambda=10^{-4}}$ & $98.50$ & $74.05$ & $66.80$ & $44.07$ & $18.94$ & $1.97$ \\
    \bottomrule
  \end{tabular}
\end{table*}

\clearpage

\section{Gradient dynamics of attention matrices}
\label{sec:gradient_dynamic_analyses}

In \Cref{fig:attn_grad_norms}, we report the layer-wise gradient spectral norms (averaged across heads) for the Q, K, and V matrices of a ViT-B trained on CIFAR-100. We evaluate three configurations considered optimal across all setups (performance metrics are detailed in \Cref{tab:results_cifar100}): $\texttt{Baseline}$, $\texttt{Specformer}$ with parameters ${(0, 0, 10^{-2})}$, and $\texttt{JaSMin}$ with $k=10$ and $\lambda=10^{-3}$. In these settings, neither $\texttt{Specformer}$ nor $\texttt{JaSMin}$ exhibits gradient vanishing, mitigating a common risk in models with reduced Lipschitz constants. Furthermore, with the current hyperparameters, $\texttt{Specformer}$ increases the gradient norm of the V matrix, which is the only regularized matrix in that specific configuration. Finally, we include $\texttt{JaSMin}_{k=10, \lambda=10^{-1}}$, which yields poor classification quality, to demonstrate that a strong regularization coefficient significantly alters the norm dynamics. The exact nature of this behavior remains a subject for future investigation.

\begin{figure}[h!tp]
    \centering
    \includegraphics[width=1\linewidth]{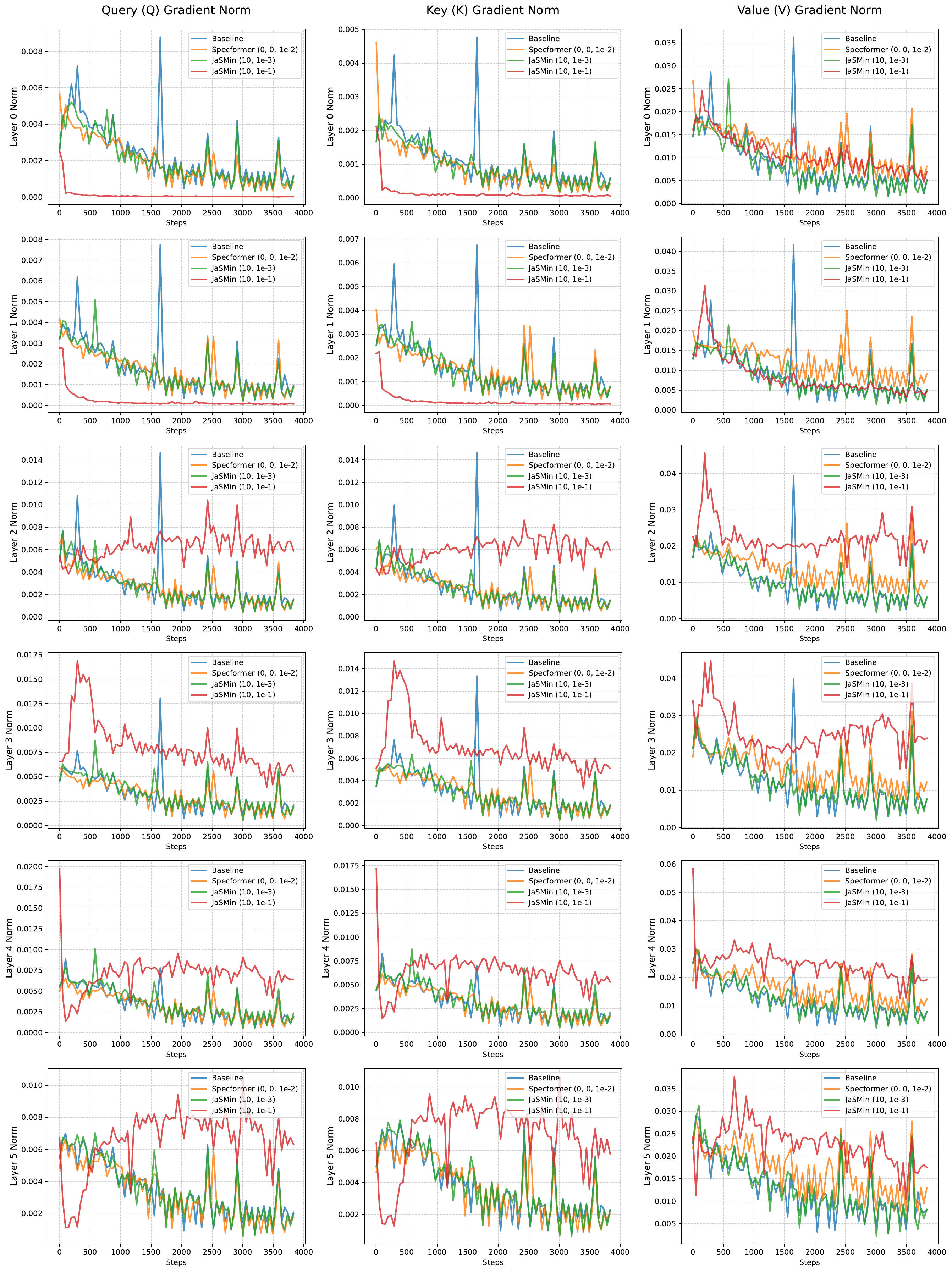}  
    \caption{Training gradient spectral norm dynamics. Results are shown for the first six layers comparing the $\texttt{Baseline}$ (blue), $\texttt{Specformer}_{(0, 0, 10^{-2})}$ (orange), $\texttt{JaSMin}_{k=10, \lambda=10^{-3}}$ (green) and $\texttt{JaSMin}_{k=10, \lambda=10^{-1}}$ (red). Only 6 of the 12 layers are plotted for visual clarity.}
    \label{fig:attn_grad_norms}
\end{figure}


\end{document}